\documentclass[11pt]{article}
\pdfoutput=1
\usepackage{enumerate}
\usepackage{pdfsync}
\usepackage[OT1]{fontenc}
\usepackage{smile}
\usepackage[colorlinks,
            linkcolor=red,
            anchorcolor=blue,
            citecolor=blue
            ]{hyperref}
\usepackage{color, colortbl}
\usepackage{fullpage}
\usepackage{hyperref}
\usepackage[protrusion=true,expansion=true]{microtype}
\usepackage{pbox}
\usepackage{setspace}
\usepackage{tabularx}
\usepackage{float}
\usepackage{wrapfig,lipsum}
\usepackage{enumitem}
\usepackage{microtype}
\usepackage{graphicx}
\usepackage{subfigure}
\usepackage{enumerate}
\usepackage{enumitem}
\usepackage{pgfplots}

\usetikzlibrary{arrows,shapes,snakes,automata,backgrounds,petri}
\usepackage{booktabs} 

\allowdisplaybreaks

\newcommand{\la}{\langle}
\newcommand{\ra}{\rangle}
\def \algname {\text{CW-OFUL}}

\def \CC {\textcolor{red}}

\usepackage{enumitem}

\usepackage{colortbl}
\definecolor{LightCyan}{rgb}{0.8, 0.9, 1}
\usepackage[utf8]{inputenc} 
\usepackage[T1]{fontenc}    
\usepackage{hyperref}       
\usepackage{url}            
\usepackage{booktabs}       
\usepackage{amsfonts}       
\usepackage{nicefrac}       
\usepackage{microtype}      
\usepackage{xcolor}         

\def \CC {}

\usepackage{enumitem}

\makeatletter
\newcommand*{\rom}[1]{\expandafter\@slowromancap\romannumeral #1@}
\makeatother
\title{\huge Nearly Optimal Algorithms for Linear Contextual Bandits with Adversarial Corruptions}

\author
{
	Jiafan He\thanks{Department of Computer Science, University of California, Los Angeles, CA 90095, USA; e-mail: {\tt jiafanhe19@ucla.edu}} 
	~~~and~~~
	Dongruo Zhou\thanks{Department of Computer Science, University of California, Los Angeles, CA 90095, USA; e-mail: {\tt drzhou@cs.ucla.edu}} 
	~~~and~~~
	Tong Zhang\thanks{Google Research and  The Hong Kong University of Science and Technology}
	~~~and~~~
	Quanquan Gu\thanks{Department of Computer Science, University of California, Los Angeles, CA 90095, USA; e-mail: {\tt qgu@cs.ucla.edu}}
}

\begin{document}
\date{}
\maketitle

\begin{abstract}
We study the linear contextual bandit problem in the presence of adversarial corruption, where the reward at each round is corrupted by an adversary, and the corruption level (i.e., the sum of corruption magnitudes over the horizon) is $C\geq 0$. The best-known algorithms in this setting are limited in that they either \CC{are computationally inefficient or} require a strong assumption on the corruption, or their regret is at least $C$ times worse than the regret without corruption. In this paper, to overcome these limitations, we propose a new algorithm based on the principle of optimism in the face of uncertainty. At the core of our algorithm is a weighted ridge regression where the weight of each chosen action depends on its confidence up to some threshold. We show that for both known $C$ and unknown $C$ cases, our algorithm with proper choice of hyperparameter achieves a regret that nearly matches the lower bounds. Thus, our algorithm is nearly optimal up to logarithmic factors for both cases. Notably, our algorithm achieves the near-optimal regret for both corrupted and uncorrupted cases ($C=0$) simultaneously. 
\end{abstract}

\section{Introduction}

We study linear contextual bandits with adversarial corruptions. 
At each round, the agent observes an decision set provided by the environment, and selects an action from the decision set. Then an adversary \emph{corrupts} the reward of the action selected by the agent. The agent then receives the corrupted reward of the selected action and proceeds until $K$ rounds. The agent's goal is to minimize the regret $\text{Regret}(K)$, which is the difference between the optimal accumulated reward and the selected accumulated reward. This problem can be regarded as a combination of the two classical bandit problems, \emph{stochastic bandits} and \emph{adversarial bandits} \citep{lattimore2018bandit}. 
In practice, the contextual bandits with adversarial corruptions can describe many popular decision-making problems such as pay-per-click advertisements with click fraud \citep{lykouris2018stochastic} and recommendation system with malicious users \citep{deshpande2012linear}.

\citet{lykouris2018stochastic} first studied the multi-armed bandit with adversarial corruptions. Specifically, let $C$ denote the \emph{corruption level} which is the sum of the corruption magnitudes at each round. \citet{lykouris2018stochastic} proposed an algorithm with a  regret that is $C$ times worse than the regret without corruption.  Later, \citet{gupta2019better} proposed an improved algorithm whose regret consists of two terms: a \emph{corruption-independent} term that matches the optimal regret for multi-armed bandit without corruption, and a \emph{corruption-dependent} term that is linear in $C$ and independent of $K$, i.e., $\text{Regret}(K) = o(K) + O(C)$. The lower bound proved in \citet{gupta2019better} suggests that the linear dependence on $C$ is near-optimal. Such a regret structure reveals an desirable property of  corruption-robust bandit algorithms, that is, the algorithm should perform nearly the same as the bandit algorithms without corruption when the corruption level $C$ is small or diminishes.

Based on the above observation, a natural question arises: 
\begin{center}
   \CC{Can we design computationally efficient algorithms for linear contextual bandits with corruption that can attain the best possible regret, similar to those in multi-armed bandits?}
\end{center}
Some previous works have attempted to answer this  question for the simpler stochastic linear bandit setting, where the decision sets at each round are identical and finite. \citet{li2019stochastic} studied the stochastic linear bandits and proposed an instance-dependent regret bound. Later on, \citet{bogunovic2021stochastic} studied the same problem and proposed an algorithm that achieves a regret with the corruption term depending on $C$ linearly and on $K$ logarithmically. However, these algorithms are limited to the stochastic linear bandit setting since their algorithm design highly relies on the experiment design and arm-elimination techniques that require a multiple selection of the same action and can only handle finite decision set. They are not applicable to contextual bandits, where the decision set is changing over time and can even be infinite. For the more general linear contextual bandit setting, \citet{bogunovic2021stochastic} proved that a simple greedy algorithm based on linear regression can attain an ideal corruption term that has a linear dependence on $C$ and a logarithmic dependence on $K$, under a stringent diversity assumption on the contexts. \citet{lee2021achieving} proposed an algorithm and the corruption term in its regret depends on $C$ linearly and on $K$ logarithmically, but only holds for the restricted case when the corruption at each round is a linear function of the action. Without special assumptions on the contexts or corruptions, \citet{zhao2021linear, ding2021robust} proposed a variant of the OFUL algorithm \citep{abbasi2011improved} and its regret has a corruption term depending on $K$ polynomially. 
\CC{Recently, \citet{wei2022model} proposed a Robust VOFUL algorithm that achieves a regret with a corruption term linearly dependent on $C'$\footnote{\CC{In \citet{wei2022model}, the adversary adds corruption to all actions in the decision set before observing the agent's action and they define the corruption level $C'$ as the maximum corruption over the decision set. See Remark \ref{remark:discussion} for the formal definition and a more detailed discussion.}} and only logarithmically dependent on $K$. However, Robust VOFUL is computationally inefficient since it needs to solve a maximization problem over a nonconvex confidence set that is defined as the intersection of exponential number of sets, and its regret has a loose dependence on context dimension $d$. 
In addition, \citet{wei2022model} also proposed a Robust OFUL algorithm and provided a regret guarantee that has a linear dependence on a different notion of corruption level $C_r$\footnote{\CC{The corruption level $C_r$ is defined as $C_r=\sqrt{K \sum_{k=1}^K c_k^2}$, where $c_k\geq 0$ is the corruption magnitude at round $k$. As a comparison, $C=\sum_{k=1}^K |c_k|$. In the worst case, $C_r=O(\sqrt{K}C)$ and therefore the corruption term in the regret of Robust OFUL will depend on $K$ polynomially.}}, which is strictly larger than the corruption level $C$ considered in the previous work and the current paper.
}
Thus, the above question remains open. 

In this paper, we give an affirmative answer to the above question. We summarize our contributions as follows. 
\begin{itemize}[leftmargin = *]
\item We propose a computationally \CC{efficient} algorithm based on the principle of optimism in the face of uncertainty \citep{abbasi2011improved}, named Confidence-Weighted OFUL (CW-OFUL). At the core of our algorithm is a weighted ridge regression where the weight of each chosen arm is adaptive to its confidence, which is defined as the truncation of the inverse exploration bonus. Intuitively, such a weighting strategy prevents the algorithm from exploiting the contexts whose rewards are more likely corrupted by a large amount. 

    \item  For the case when the corruption level $C$ is known to the agent, we show that the proposed algorithm enjoys a regret $\text{Regret}(K) = \tilde O(d\sqrt{K} + dC)$, where $d$ is the dimension of the contexts, $C$ is the corruption level and $K$ is the number of total iterations. The first term matches the regret lower bound of linear contextual bandits without corruption $\Omega(d\sqrt{K})$ \citep{lattimore2018bandit}. The second term matches the lower bound on the corruption term in regret $\Omega(dC)$  \citep{bogunovic2021stochastic}. They together suggest that our algorithm is not only robust but also near-optimal up to logarithmic factors. 
    \item For the case when the corruption level $C$ is unknown to the agent, we show that $\algname$ enjoys an $\tilde O(d\sqrt{K})$ regret for the case $C \leq \sqrt{K}$, with proper choice of the hyperparameter. Surprisingly, by proving a lower bound on the regret, we show that our regret upper bound is already optimal for all algorithms that achieve a near-optimal regret bound for uncorrupted bandits.
\end{itemize}

We compare our regret bounds with previous ones in Table \ref{table:11}. We can see that our algorithm matches the lower bound up to logarithmic factors in both known $C$ and unknown $C$ cases, and therefore is nearly optimal.
\newcolumntype{g}{>{\columncolor{LightCyan}}c}
\begin{table*}[ht]
\caption{Comparisons of regrets for corrupted linear contextual bandits.}\label{table:11}
\centering
\begin{tabular}{ggg}
\toprule
\rowcolor{white} Algorithm & Regret& $C$\\
\midrule
\rowcolor{white} Robust weighted OFUL & & \\
\rowcolor{white} \small{\citep{zhao2021linear}}  & \multirow{-2}{*}{$\tilde O(d\sqrt{K} + dC\sqrt{K})$} & \multirow{-2}{*}{Known} \\
\rowcolor{white} Robust OFUL & & \\
\rowcolor{white} \small{\citep{wei2022model}}  & \multirow{-2}{*}{$\tilde O(d\sqrt{K}+C_r)$} & \multirow{-2}{*}{Known} \\
\rowcolor{white} Robust VOFUL & & \\
\rowcolor{white} \small{\citep{wei2022model}}  & \multirow{-2}{*}{$\tilde O(d^{4.5}\sqrt{K} +d^4C')$} & \multirow{-2}{*}{Known} \\
$\algname$ & & \\
\small{(Theorem \ref{theorem:known-C})}& \multirow{-2}{*}{$\tilde O(d\sqrt{K} + dC)$} & \multirow{-2}{*}{Known}\\
\rowcolor{white} Lower bound & & \\
\rowcolor{white} \small{\citep{lattimore2018bandit, bogunovic2021stochastic}}& \multirow{-2}{*}{$\Omega(d\sqrt{K} + dC)$} & \multirow{-2}{*}{Known}\\
\midrule
\rowcolor{white} Multi-level weighted OFUL & &  \\
\rowcolor{white} \small{\citep{zhao2021linear}}& \multirow{-2}{*}{$\tilde O(dC^2\sqrt{K}),\ C = \Omega(1)$} & \multirow{-2}{*}{Unknown} \\
\rowcolor{white} Greedy & & \\
\rowcolor{white} \small{\citep{bogunovic2021stochastic}}  & \multirow{-2}{*}{$\tilde O\big((\sqrt{dK}+C)/\lambda_0\big)^3$} & \multirow{-2}{*}{Unknown} \\
\rowcolor{white} COBE$+$OFUL & & \\
\rowcolor{white} \small{\citep{wei2022model}}  & \multirow{-2}{*}{$\tilde O(d\sqrt{K}+C_r)$} & \multirow{-2}{*}{Unknown} \\
\rowcolor{white} COBE$+$VOFUL & & \\
\rowcolor{white} \small{\citep{wei2022model}}  & \multirow{-2}{*}{$\tilde O(d^{4.5}\sqrt{K} +d^4C')$} & \multirow{-2}{*}{Unknown} \\
 $\algname (\bar{C}=\sqrt{K})$ & $\tilde O(d\sqrt{K}),\ C \leq \sqrt{K}$&  \\
 \small{(Theorem \ref{thm:unknown})}& $O(K),\ C \geq \sqrt{K}$ & \multirow{-2}{*}{Unknown}  \\
 Lower bound\footnotemark[4] & $\Omega(d\sqrt{K}),\ C\leq \sqrt{K}$ & \\
\small{(\citealt{lattimore2018bandit}, Theorem \ref{thm:lower_uc})} & $\Omega(K),\ C \geq \sqrt{K}$ & \multirow{-2}{*}{Unknown} \\
\bottomrule
\end{tabular}
\end{table*}

\subsection{Additional Related Work}

\noindent\textbf{Bandits with Misspecification.}
Bandits with misspecification can be seen as a special case of bandit with adversarial corruption since it is corrupted relative evenly at each round. 
Let $\epsilon$ be the misspecification level. \citet{ghosh2017misspecified} firstly studied the stochastic linear bandits and proved a sublinear regret when $\epsilon$ is small. \citet{lattimore2019learning} studied the stochastic linear bandit setting under milder assumptions. With the knowledge of $\epsilon$, they proposed an algorithm with an $\tilde O(\sqrt{dK \log (N)} + \epsilon\sqrt{d} K)$ regret, where $d$ is the dimension of the contextual vector, $N$ is the number of arms. Their regret bound matches their proved lower bound up to logarithmic factors. \citet{foster2020adapting} further considered the more general linear contextual bandits with misspecification when $\epsilon$ is unknown to the agent, and proposed an algorithm equipped with a CORRAL meta algorithm \citep{agarwal2017corralling} to deal with the unknown $\epsilon$. Their algorithm enjoys an $\tilde O(d\sqrt{K}+ \epsilon\sqrt{d} K)$ regret. \citet{krishnamurthy2021adapting} proposed an algorithm without using a meta algorithm which has the same order of regret as \citet{foster2020adapting}. Our algorithm can be directly applied to the misspecification setting by choosing the corruption level $C$ to be $K\epsilon$, which immediately gives us an $\tilde O(d\sqrt{K} + dK\epsilon)$ regret upper bound.
\footnotetext[3]{\CC{Greedy \cite{bogunovic2021stochastic} assumes that each arm in the decision set at each round is sampled from a distribution that satisfies $(r,\lambda_0)$-diverse property \citep{kannan2018smoothed}. A distribution $\cD$ is $(r,\lambda_0)$-diverse if for any $\ab = \bmu + \bxi$ with $\bmu \in \RR^d$ and $\bxi \sim \cD$,  $\lambda_{\text{min}}(\EE_{\bxi \in \cD}[\ab\ab^\top|\btheta^\top \bxi \geq b]) \geq \lambda_0$ holds for all $\btheta \in \RR^d$ and $b \in \RR$ satisfying $b \leq r\|\btheta\|_2$.}}
\footnotetext[4]{\CC{The lower bound under a large corruption level $C\ge \sqrt{K}$ only holds for algorithms that can achieve near-optimal regret for uncorrupted bandits.  
It is possible for an algorithm that does not achieve the optimal regret for uncorrupted bandits (e.g., $R_K = O(K^{0.75})$) to achieve a sub-linear regret in the presence of corruptions.}}

\noindent\textbf{Bandits with Adversarial Rewards.} 
There exists a large body of literature on the problems of adversarial multi-armed bandits \citep{auer2002nonstochastic, bubeck2012regret}. 
There is also a line of works trying to design algorithms that can achieve near-optimal regret bounds for both stochastic bandits and adversarial bandits simultaneously \citep{pmlr-v23-bubeck12b, pmlr-v32-seldinb14, pmlr-v49-auer16, pmlr-v65-seldin17a, pmlr-v89-zimmert19a, lee2021achieving}. However, most of these algorithms focus on the general adversarial reward setting without specifying the total amount of corruption. One of the notable exceptions is \citet{lee2021achieving}, which assumed that the adversarial corruptions are generated through the inner product of an adversarial vectors and the contextual vector. As a comparison, our algorithm and result do not need such additional assumption on the structure of the corruption. Our algorithm can be applied to both corrupted and uncorrupted settings with different choices of hyperparameters, and achieves a near-optimal regret for both cases. 

\noindent\textbf{Notation} 
We use lower case letters to denote scalars, and use lower and upper case bold face letters to denote vectors and matrices respectively. We denote by $[n]$ the set $\{1,\dots, n\}$. For a vector $\xb\in \RR^d$ and a positive semi-definite matrix $\bSigma\in \RR^{d\times d}$, we denote by $\|\xb\|_2$ the vector's $\ell_2$ norm and by $\|\xb\|_{\bSigma}=\sqrt{\xb^\top\bSigma\xb}$ the Mahalanobis norm. For two positive sequences $\{a_n\}$ and $\{b_n\}$ with $n=1,2,\dots$, 
we write $a_n=O(b_n)$ if there exists an absolute constant $C>0$ such that $a_n\leq Cb_n$ holds for all $n\ge 1$ and write $a_n=\Omega(b_n)$ if there exists an absolute constant $C>0$ such that $a_n\geq Cb_n$ holds for all $n\ge 1$. We use $\tilde O(\cdot)$ to further hide the polylogarithmic factors. We use $\ind\{\cdot\}$ to denote the indicator function. 

\section{Preliminaries}

In this section, we introduce the setting of linear contextual bandit with adversarial corruption. 

\noindent\textbf{Linear contextual bandit with corruption.}
We define linear contextual bandits with corruption as follows: at the beginning of each round $k\in [K]$, the agent receives a decision set $\cD_k \in \RR^d$ from the environment and it chooses an action (i.e., arm, contextual vector) $\xb \in \cD_k$. After choosing the action $\xb_k$ at round $k$, the environment generates the corresponding $r_k$ based on the stochastic linear model $r_k=\la \btheta^*, \xb\ra +\eta_k$, where $\btheta^* \in \RR^d$ is an unknown environment parameter and $\eta_k$ is the stochastic noise. After seeing the stochastic reward $r_k$, the adversary (i.e., attacker) introduces an adversarial corruption $c_k$ onto the reward, which may depend on the decision set $\cD_k$, action $\xb_k$, stochastic reward $r_k$.
Finally, the agent observes the corrupted reward $\hat{r}_k=\la \btheta^*, \xb\ra +\eta_k+c_k$ at round $k$. Following \citet{abbasi2011improved}, we make the following assumptions on the bandit model.
\begin{assumption}\label{ass:setup}
The linear contextual bandit satisfies the following conditions:
\begin{itemize}[leftmargin=*]
    \item At each round $k$ and any action $\xb \in \cD_k$, we have $\|\xb\|_2\leq L$.\
    \item For the unknown environment parameter $\btheta^*$, it satisfies  $\|\btheta^*\|_2\leq S$.
    \item At each round $k$, the corresponding stochastic noise $\eta_k$ is conditional $R$-sub-Gaussian, i.e.,
    \begin{align}
    \forall \lambda \in \RR,\  \EE\big[e^{\lambda\eta_k}|\xb_{1:k}, \epsilon_{1:k-1},c_{1:k-1}\big] \leq \exp(R^2\lambda^2/2). \notag
\end{align}
\end{itemize}
\end{assumption}
\noindent\textbf{Regret.} The goal of the agent is to minimize the pseudo-regret in the first $K$ rounds, which is defined as follows:
\begin{align*}
    \text{Regret}(K)= \textstyle{\sum_{k=1}^K} \max_{\xb \in \cD_k}\la \btheta^*, \xb\ra - \la \btheta^*, \xb_k\ra.
\end{align*}
\noindent\textbf{Corruption level.}
To measure the level of adversarial corruptions, we define the \emph{corruption level} as $C: = \sum_{k=1}^K|c_k|$.
With this definition, we say a linear contextual bandit problem is $C$-corrupted if and only if the corruption level is no larger than $C$.
\begin{remark}\label{remark:discussion}
The adversary in our setting and the corresponding definition of corruption level is the same as that in \citet{bogunovic2021stochastic} and slightly different from that in prior works such as \citet{lykouris2018stochastic,pmlr-v99-gupta19a,zhao2021linear}. More specifically, in these works, the adversarial corruption $c_k$ is chosen before the choice of action $\xb_k \in \cD_k$. Since the actions selected by the agent may not be deterministic, the adversary chooses different corruption $c_{k,\xb}$ for different action $\xb \in \cD_k$. With this notion of corruption, the corresponding corruption level is defined as $
    C'=\sum_{k=1}^K \max_{\xb \in \cD_k}|c_{k,\xb}|$.
As a comparison, our adversary chooses the corruption after observing the action $x_k$ and for the corruption level. We have
\begin{align*}
    C=\sum_{k=1}^K|c_{k,\xb_k}|\leq \sum_{k=1}^K \max_{\xb \in \cD_k}|c_{k,\xb}|=C',
\end{align*}
which implies that our corruption level $C$ is always no larger than the corruption level $C'$ in \citet{lykouris2018stochastic,pmlr-v99-gupta19a,zhao2021linear}.
\end{remark}

\section{Algorithms}\label{section: known C}
In this section, we review existing algorithms for linear contextual  bandits (and stochastic linear bandits) and discuss their limitations when they are applied to the adversarial corruption setting. Then we present our algorithm $\algname$ and illustrate how our algorithm design can overcome the these limitations. 

\subsection{Existing Algorithms}
We begin with reviewing the classical OFUL algorithm \citep{abbasi2011improved}. Under Assumption \ref{ass:setup}, at round $k$, OFUL estimates $\btheta^*$ by online ridge regression over all the past observed actions and rewards, i.e., 
\begin{align}
    \btheta_k&\leftarrow \argmin_{\btheta\in \RR^d}\lambda\|\btheta\|_2^2+\textstyle{\sum_{i=1}^{k-1}}\big(\btheta^{\top}\xb_i-r_i\big)^2.\label{eq:01}
\end{align}
With $\btheta_k$ in hand, OFUL constructs a confidence set for $\btheta^*$ as follows $\cC_k=\Big\{\btheta: \|\btheta_k-\btheta\|_{\bSigma_k}\leq \beta\Big\}$,
where $\beta$ is the confidence radius and $\bSigma_k = \lambda \Ib + \sum_{i=1}^{k-1}\xb_i\xb_i^\top$ is the covariance matrix of contexts $\xb_i, i = 1,\ldots, k$. Without corruption, it is known that setting $\beta = \tilde O(R\sqrt{d})$ guarantees that $\btheta^* \in \cC_k$ with high probability, which further leads to a sublinear regret $\tilde O(d\sqrt{K})$. However, with corruption, such a choice of $\beta$ is not sufficient. To see why, we take a closer look at the closed-form solution $\btheta_k$ to \eqref{eq:01}:
\begin{align}
    \btheta_k&=\bSigma_k^{-1}\sum_{i=1}^{k-1}\xb_i r_i=\bSigma_k^{-1}\sum_{i=1}^{k-1}\xb_i (\xb_i^{\top}\btheta^*+\eta_i)+\bSigma_k^{-1}\sum_{i=1}^{k-1}\xb_i c_i.\notag
\end{align}
By simple calculation and assuming $\lambda$ to be a constant, we can show that $\|\btheta_k-\btheta^*\|_{\bSigma_k}$ can be upper bounded by
\begin{align}
    \|\btheta_k-\btheta^*\|_{\bSigma_k} \leq O\bigg(\underbrace{\bigg\|\sum_{i=1}^{k-1} \xb_i \eta_i\bigg\|_{\bSigma_k^{-1}}}_{I_1} + \underbrace{\bigg\|\sum_{i=1}^{k-1} \xb_i c_i\bigg\|_{\bSigma_k^{-1}}}_{I_2}\bigg).\notag
\end{align}
The first term $I_1$ is corruption-independent and bounded by $\tilde O(R\sqrt{d })$ according to \citet{abbasi2011improved}. The challenge is to bound the second term $I_2$, which depends on the corruption. Existing approaches \citep{zhao2021linear,ding2021robust} bound $I_2$ by triangle inequality and Cauchy-Schwarz inequality, 
\begin{align}
    I_2\leq \sum_{i=1}^{k-1}\|\xb_i c_i\|_{\bSigma_k^{-1}}\leq \sum_{i=1}^{k-1}|c_i| \max_{1 \leq j \leq k-1}\|\xb_j\|_{\bSigma_k^{-1}} \leq \sum_{i=1}^{k-1} |c_i|L/\sqrt{\lambda}=O(C),\label{help:0}
\end{align}
where $C$ is the corruption level and $\max_{1 \leq j \leq k-1}\|\xb_j\|_{\bSigma_k^{-1}}$ is bounded by the crude upper bound $L/\sqrt{\lambda}$. Unfortunately, such a bound makes the confidence radius be the order of $O(R\sqrt{d} + C)$, which eventually leads to an  term $O(C\sqrt{K})$ in the regret, which is $C$ times worse than the regret without corruption. 

In order to obtain a tighter bound of $I_2$, for stochastic linear bandits, \citet{bogunovic2021stochastic} proposed a Robust Phase Elimination (RPE) algorithm, which employs \emph{optimal design} \citep{lattimore2018bandit} to select the arms. In this setting, the decision set is finite and fixed over time, i.e., $\cD_k = \cD$ for all $k \in [K]$ and $|\cD|\leq \infty$. More specifically, RPE divides the time horizon into several phases. Within each phase, RPE performs linear regression on a multiset $\cA \subset \cD$, which is the G-optimal design of $\cD$. Here the multiset means $\cA$ has duplicate elements. Let $\bSigma$ be the covariance matrix defined over $\cA$, then the following upper bound holds \citep{lattimore2018bandit}: 
\begin{align}
    \forall \xb \in \cD,\ \|\xb\|_{\bSigma^{-1}} = O\big(|\cA|^{-1/2}\big).\label{help:1}
\end{align}
By choosing a large enough $|\cA|$, \eqref{help:1} provides a \emph{uniformly small} upper bound for $\max_{1 \leq j \leq k-1}\|\xb_j\|_{\bSigma_k^{-1}}$ for any $k$. Substituting \eqref{help:1} back into \eqref{help:0} with $|\cA| = O(C^2)$, we can show that $I_2$ is bounded by some small constant, which therefore eliminates the $O(C\sqrt{K})$ term in the final regret. Although the optimal design-based approach RPE \citep{bogunovic2021stochastic} successfully eliminates the multiplicative term $C\sqrt{K}$, it is not applicable to our linear contextual bandit setting: 
(1) it needs to select a \emph{multiset} from the decision set, which is impossible for the general contextual bandit setting; (2) the complexity of optimal design introduces some additional quadratic term $C^2$ in their final regret, which  makes their algorithm non-optimal (See \citet{bogunovic2021stochastic} for more details). 

\begin{algorithm*}[t]
	\caption{$\algname$}\label{algorithm-bandit-known}
	\begin{algorithmic}[1]
    \REQUIRE Regularization parameter $\lambda$, confidence radius $\beta$ and threshold parameter $\alpha$
	\FOR {round $k=1,2,..$}
	   
	      \STATE Set $\bSigma_{k}=\lambda \Ib+\sum_{i=1}^{k-1}w_i \xb_i\xb_i^\top$\label{algorithm1-line4}
	            \STATE Set $\bbb_{k}=\sum_{i=1}^{k-1}w_i \xb_i r_i$ and $\btheta_{k}= \bSigma_{k}^{-1}\bbb_{k}$\label{algorithm1-line5}

	    \STATE Receive the decision set $\cD_k$
	    \STATE Choose action $\xb_k\leftarrow \argmax_{\xb \in \cD_k}  \btheta_{k}^{\top}\xb+\beta\sqrt{\xb^{\top}\bSigma_{k}^{-1} \xb}$\label{algorithm1-line8}
	    \STATE Set $w_k = \min\{1, \alpha/\|\xb_k\|_{\bSigma_k^{-1}}\}$\label{algorithm1-line6}
	\ENDFOR
	\end{algorithmic}
\end{algorithm*}

\subsection{Our Algorithm}
As we have seen before, it is pivotal to bound the corruption-dependent term $I_2$ tightly. To overcome the limitations of existing approaches, we propose a fundamentally new approach and present our $\algname$ in Algorithm \ref{algorithm-bandit-known}. At a high level, Algorithm \ref{algorithm-bandit-known} is an extension of the OFUL algorithm \citep{abbasi2011improved}, which is also based on the principle of optimism in the face of uncertainty. 

Our algorithm assigns a weight $w_k$ to each selected action $\xb_k$. More specifically, at round $k$, we use the following weighted ridge regression to estimated the unknown vector $\btheta^*$:
\begin{align}
\btheta_k\leftarrow \argmin_{\btheta\in \RR^d}\lambda\|\btheta\|_2^2+\textstyle{\sum_{i=1}^{k-1}}w_i\big(\btheta^{\top}\xb_i-r_i\big)^2.\label{eq:gu0001}
\end{align}
The closed-form solution to the above optimization problem is displayed in Line \ref{algorithm1-line5} of Algorithm \ref{algorithm-bandit-known}. While weighted ridge regression is not new and has been used in prior work on bandits \citep{kirschner2018information, zhou2021nearly,russac2019weighted}, the setting, motivation and the choice of weight are fundamentally
different. More specifically, we choose the weight as the \emph{truncation} of the inverse exploration bonus, which is $w_k = \min\Big\{1, \alpha/\|\xb_k\|_{\bSigma_k^{-1}}\Big\}$.
Here $\alpha>0$ is a threshold parameter. We can see that for action $\xb_k$ with a large exploration bonus $\|\xb_k\|_{\bSigma_k^{-1}}$ (low confidence), $\algname$ will assign a small weight to it to avoid the potentially large regret caused by both the stochastic noise and the adversarial corruption. On the other hand, for the action with a small exploration bonus (high confidence), $\algname$ will assign a large weight to it (it can be as large as $1$). Another interesting observation is that by setting $\alpha$ to be sufficiently large, the weight will become $1$ for every action, and $\algname$ will degenerate to OFUL \citep{abbasi2011improved}.

\CC{As a comparison, \citet{kirschner2018information, zhou2021nearly} used the inverse of the noise variance as the weight to normalize the noise and derived tight variance-dependent regret guarantees. \citet{russac2019weighted} set
the weight as a geometric sequence to perform moving average to deal with the non-stationary environment.} 

To see how our choice of weight can lead to tighter regret, we first write down the closed-form solution to \eqref{eq:gu0001} 
\begin{align}
    \btheta_k
    &=\bSigma_k^{-1}\sum_{i=1}^{k-1}w_i\xb_i (\xb_i^{\top}\btheta^*+\eta_i)+\sum_{i=1}^{k-1}\bSigma_k^{-1}w_i\xb_i c_i,\notag
\end{align}
where the covariance matrix $\bSigma_k=\lambda\Ib+\sum_{i=1}^{k-1}w_i\xb_i\xb_i^{\top}$. With some calculation and assuming $\lambda$ to be a constant, we can obtain
\begin{align}
    \|\btheta_k-\btheta^*\|_{\bSigma_k} \leq O\bigg(\underbrace{\bigg\|\sum_{i=1}^{k-1} w_i\xb_i \eta_i\bigg\|_{\bSigma_k^{-1}}}_{I_1} + \underbrace{\bigg\|\sum_{i=1}^{k-1} w_i\xb_i c_i\bigg\|_{\bSigma_k^{-1}}}_{I_2}\bigg).\notag
\end{align}
$I_1$ is the corruption-independent term and can still be bounded by $\tilde O(R\sqrt{d})$ according to \citet{abbasi2011improved}. For $I_2$, we have
\begin{align*}
     \bigg\|\sum_{i=1}^{k-1}w_i\xb_i c_i\bigg\|_{\bSigma_k^{-1}}\leq \sum_{i=1}^{k-1}|c_i|w_i \big\|\xb_i\big\|_{\bSigma_k^{-1}}\leq \sum_{i=1}^{k-1}|c_i| \alpha=C\alpha,
\end{align*}
It is evident that with our carefully designed weight, the corruption-dependent term $I_2$ can be uniformly bounded by some constant $C\alpha$, the same as that in \citet{bogunovic2021stochastic}. Therefore, by setting $\alpha$ to be sufficiently small, our $\algname$ can get rid of the $C\sqrt{K}$ term in the final regret.

\section{Main Results}\label{section:5}
In this section, we present the main theoretical guarantees of $\algname$. 

\subsection{Known Corruption Level $C$}\label{limit}
We first consider the case when $C$ is known to the agent. In this case, we choose $\alpha = R\sqrt{d}/C$. The following lemma characterizes the estimation error of $\btheta_k$ with respect to $\btheta^*$, which is a formal summary of our discussion in Section \ref{section: known C}.
\begin{lemma}\label{lemma:weighted-concentration}
Suppose that Assumption \ref{ass:setup} holds. For any $0 < \delta < 1$ and corruption budget $C\geq 0$, set the confidence radius $\beta=R\sqrt{d\log \big((1+KL^2/\lambda)/\delta\big)}+\sqrt{\lambda}S+\alpha C$ in Algorithm \ref{algorithm-bandit-known}, then with probability at least $1-\delta$, for every round $k$, the estimator $\btheta_k$ satisfies that $\|\btheta_k-\btheta^*\|_{\bSigma_k}\leq \beta$.
\end{lemma}

The following theorem provides the regret bound of Algorithm  \ref{algorithm-bandit-known}. 
\begin{theorem}\label{theorem:known-C}
Suppose that Assumption \ref{ass:setup} holds. For any $0 < \delta < 1$ and corruption budget $C\geq 0$,  set the confidence radius $\beta$ in Algorithm \ref{algorithm-bandit-known} as follows:
\begin{align}
    \beta=R\sqrt{d\log \big((1+KL^2/\lambda)/\delta\big)}+\alpha C+\sqrt{\lambda}S. \notag
\end{align}
Then with probability at least $1-\delta$, its regret in the first $K$ rounds is upper bounded by
\begin{align*}
   \text{Regret}(K)&=  O\bigg(dR\sqrt{K\log^2\big((1+KL^2/\lambda)/\delta\big)}+\alpha C\sqrt{dK\log^2\big((1+KL^2/\lambda)/\delta\big)}\notag\\
      &\qquad +S\sqrt{d\lambda K \log(1+KL^2/\lambda)} + \frac{Rd^{1.5}}{\alpha}\times \sqrt{\log^{3}\big((1+KL^2/\lambda)/\delta\big)}\notag\\
      &\qquad + \frac{dS\sqrt{\lambda}}{\alpha}\times\sqrt{\log^{2}\big((1+KL^2/\lambda)/\delta\big)} + d C \sqrt{\log^{2}\big((1+KL^2/\lambda)/\delta\big)} \bigg).
\end{align*}
\CC{In addition, if choosing $\alpha=(R\sqrt{d}+\sqrt{\lambda}S)/C$ and $\lambda=R^2/S^2$, its regret can be upper bounded by
}
\begin{align*}
    \text{Regret}(K)&=  \tilde O(d\sqrt{K}+dC).
\end{align*}
\end{theorem}
A few remarks about Theorem \ref{theorem:known-C} are in order.
\begin{remark}
  Compared with the $\tilde O(d\sqrt{K} + dC\sqrt{K})$ regret proved in \citet{zhao2021linear,ding2021robust}, our algorithm improves the multiplicative dependence on corruption level $C$ to additive dependence. In particular, $\algname$ achieves the same order of regret as the uncorrupted setting when $C = O(\sqrt{K})$, and it attains a sublinear regret as long as $C = o(K)$. In sharp contrast, the algorithm proposed in \citet{zhao2021linear} achieves the same order of regret as the uncorrupted setting only when $C = O(1)$, and has a sublinear regret only when $C = o(\sqrt{K})$. 
\end{remark}
\begin{remark}\label{rmk:test}
\CC{We also compare our result with that in \citet{wei2022model}. The Robust$+$OFUL algorithm in \citet{wei2022model} achieves an $\tilde O(d\sqrt{K}+C_r)$ regret with $C_r = \sqrt{T\sum_{k=1}^K c_k^2}$, which will degenerate to $\tilde O(d\sqrt{K} +d\sqrt{K}C)$ in the worst case. Their regret guarantee is always worse than ours when $C<\sqrt{K}$. In addition, according to the discussion in Remark \ref{remark:discussion}, Theorem \ref{theorem:known-C} also implies an $\tilde O\big(\sqrt{dK}+dC'\big)$ regret under the notion of the corruption level $C'$. In contrast, 
the Robust VOFUL algorithm in \citet{wei2022model} has an $\tilde O(d^{4.5}\sqrt{K} +d^4C')$ regret, which is also inferior to our regret. Furthermore, Robust VOFUL is computationally inefficient.}
\end{remark}

\begin{remark}
We further compare our result with previous additive regrets derived for stochastic linear bandits. Let $\cD_k = \cD$ be the decision set. Compared with the $O(\sqrt{dK\log |\cD|}+Cd^{3/2})$ regret for stochastic linear bandit with corruption derived in \citet{bogunovic2021stochastic}, our regret improves the corruption term by a factor of $\sqrt{d}$. Note that the $\sqrt{d}$ difference in the leading $\sqrt{K}$ term between our regret and theirs is caused by the fact that \citet{bogunovic2021stochastic} considered the finite-arm setting, while we consider the infinite-arm setting. Our algorithm will have the same regret as theirs when $|\cD| = O(\text{exp}(d))$. 
\end{remark}

\begin{remark}
For the uncorrupted setting where $C = 0$, Theorem \ref{theorem:known-C} suggests that the threshold parameter $\alpha$ should be set to infinity. Then by Line \ref{algorithm1-line6} in Algorithm \ref{algorithm-bandit-known}, each weight $w_k$ becomes $1$, and $\algname$ degenerates to OFUL. Meanwhile, the regret in Theorem \ref{algorithm-bandit-known} also becomes $\tilde O(d\sqrt{K})$ that matches the regret of OFUL \citep{abbasi2011improved}. 
\end{remark}

\noindent\textbf{Lower bound.}
Next we refer to two existing lower bound results to show that when $C$ is known, our $\tilde O(d\sqrt{K} + dC)$ regret is optimal up to logarithmic factors. The first proposition shows that the $\tilde O(d\sqrt{K})$ corruption-independent term in our regret is near-optimal.
\begin{proposition}[Theorem 24.2, \citealt{lattimore2018bandit}]\label{prop:uc}
Assume $d \leq 2K$, $R   = 1$ and $\cD_k = \{\|\xb\|_2 \leq 1\}$ for all $k\geq 1$. Then for any algorithm, there exists a environment parameter vector $\btheta^* \in \RR^d$ satisfying $\|\btheta^*\|_2^2 = d^2/(48K)$ such that $\EE(\text{Regret}(K)) \geq d\sqrt{K}/(16\sqrt{3})$. 
\end{proposition}

The second proposition suggests that the $O(dC)$ corruption term in our regret is optimal.
\begin{proposition}[Theorem 3, \citealt{bogunovic2021stochastic}]\label{prop:uc1}
For any dimension $d$, for any algorithm that has the knowledge of $C$, there exists an instance satisfying with probability at least $0.5$, $\text{Regret}(K)  = \Omega(dC)$.
\end{proposition}
Combining Propositions \ref{prop:uc} and \ref{prop:uc1}, we can conclude that for any algorithm, there exists a corrupted bandit instance such that the algorithm suffers at least $\Omega(\max\{d\sqrt{K}, dC\})$ regret. Such a lower bound matches our upper bound up to logarithmic factors. Therefore, our algorithm is nearly optimal.

\noindent\textbf{Misspecified linear bandits.} 
We consider the misspecified linear bandit setting which assumes that the corruption at each round is uniformly bounded by $\epsilon$. Clearly, the misspecified linear bandit is a special case of corrupted linear contextual bandit with $C = K\epsilon$. Theorem \ref{theorem:known-C} suggests that a direct application of our algorithm to this special setting incurs an $\tilde O(d\sqrt{K} + dK\epsilon)$ regret, which differs from the near-optimal regret $\tilde O(d\sqrt{K} + \sqrt{d}K\epsilon)$ \citep{lattimore2019learning, foster2020adapting} by a $\sqrt{d}$ factor on the corruption term. Whether our algorithm is able to achieve the near-optimal regret for both misspecified linear bandit and corrupted linear contextual bandit simultaneously remains an open question.

\noindent\textbf{Instance-dependent regret bounds.}
$\algname$ also enjoys an instance-dependent regret bound. Due to space limit, we defer it to Appendix \ref{sec:instance}.

\subsection{Unknown Corruption Level $C$}
Now we consider the case when $C$ is unknown. Our solution is quite simple for this case: we introduce a tuning parameter $\bar C$, which can be viewed as an estimate of $C$, and select the threshold parameter $\alpha$ as Theorem \ref{theorem:known-C} suggests. The following theorem gives the regret upper bound of $\algname$ for the unknown $C$ case. 

\begin{theorem}\label{thm:unknown}
Under the same conditions of Theorem \ref{theorem:known-C} except that we set $\alpha = (R\sqrt{d}+\sqrt{\lambda}S)/\bar{C}$ with $\bar{C}$ being an estimated corruption level, $\lambda = R^2/S^2$ and $\beta=2R\sqrt{d\log \big((1+KL^2/\lambda)/\delta\big)}+2\sqrt{\lambda}S$ in Algorithm \ref{algorithm-bandit-known}. 
Its regret can be upper bounded by
\begin{itemize}[leftmargin = *]
    \item If $0 \leq C \leq \bar{C}$, then with probability at least $1-\delta$, we have $\text{Regret}(K) 
    = \tilde O(dR\sqrt{K}+d\bar{C})$.
     \item If $C > \bar{C}$, we have $\text{Regret}(K) = O(K)$.
\end{itemize}
In addition, if we set $\bar{C}=\sqrt{K}$, then when $0\leq C\leq \sqrt{K}$, the regret is upper bounded by $\tilde O(d\sqrt{K})$. 
\end{theorem}

\begin{remark}
\citet{zhao2021linear} proposed an $\tilde O(C^2d\sqrt{K})$ regret with unknown $C = \Omega(1)$. Compared with their result, our regret (with $\bar{C}=\sqrt{K}$) is strictly better in the corruption term. \citet{bogunovic2021stochastic} proposed an $\tilde O(\sqrt{dK \log|\cD|} + Cd^{1.5} + C^2)$ regret for the stochastic linear bandit with unknown $C$, in the regime $C = \tilde O(\sqrt{K}/d)$, where $\cD$ is the finite decision set. Such a regret becomes $\tilde O(d\sqrt{K} + Cd^{1.5} + C^2)$ when the size of $\cD$ becomes exponentially large in $d$ or even infinite. Compared with their regret, our regret is not only smaller, but also holds for a wider regime (i.e., $C = O(\sqrt{K})$). \CC{Compared with the greedy algorithm in \citet{bogunovic2021stochastic}, our result does not rely on the stringent $(r,\lambda_0)$-diverse property assumption on the contexts.}
\end{remark}
\begin{remark}
\CC{We also compare our result (choosing $\bar{C}=\sqrt{K}$) with those in \citet{wei2022model} for the unknown $C$ case. \citet{wei2022model} proposed a COBE$+$OFUL algorithm with an $\tilde O(d\sqrt{K}+C_r)$ regret, and a COBE$+$VOFUL algorithm with an $\tilde O(d^{4.5}\sqrt{K} +d^4C')$ regret, analogous to their results for the known $C$ case discussed in Remark \ref{rmk:test}. Our CW-OFUL enjoys a better regret than COBE$+$OFUL for all $C$, and it is better than COBE$+$VOFUL for $C<\sqrt{K}$. In addition, for the modified notion of corruption level $C'$, if we choose the basic algorithm in COBE \citep{wei2022model} as our CW-OFUL algorithm, then Theorem 3 in \citet{wei2022model} suggests that COBE+CW-OFUL can deal with unknown corruption level $C'$ and obtained an $\tilde O(d\sqrt{K}+dC')$ regret guarantee, which matches the regret of CW-OFUL algorithm with known corruption level $C'$. Note that COBE$+$VOFUL is also computationally inefficient.
} 

\end{remark}
With $\bar{C}=\sqrt{K}$, for the case when $0 \leq C \leq \sqrt{K}$, our regret result is already near-optimal, due to the lower bound for the uncorrupted bandit in Proposition \ref{prop:uc}. Now we show that our $O(K)$ bound, seemingly trivial, is actually optimal for a large class of bandit algorithms. In detail, the following theorem provides a lower bound result for any algorithm for the unknown $C$ case. This is an extension of the lower bound result in \citet{bogunovic2021stochastic} from  $d=2$ to general $d$.

\begin{theorem}\label{thm:lower_uc}
For any algorithm, let $R_K$ be an upper bound of $\text{Regret}(K)$ such that for any bandit instance satisfying Assumption \ref{ass:setup} with $C = 0$, \CC{it satisfies the $\EE\big[\text{Regret}(K)\big] \leq R_K  \leq O(K)$, where the expectation is with respect to the possible randomness of the algorithm and the stochastic noise. Then for the general case with $C = \Omega(R_K/d)$, such an algorithm will have $\EE\big[\text{Regret}(K)\big]  = \Omega(K)$. }
\end{theorem}
\begin{remark}
 Consider a class of algorithms $\cA$ whose worst-case regret is $R_K$ in the uncorrupted case. Here we only need to consider $\Omega(d\sqrt{K}) \leq R_K \leq O(K)$, since for any algorithm, $\Omega(d\sqrt{K})$ is the lowest possible worst-case regret \citep{lattimore2018bandit} and $O(K)$ is the highest possible regret. We first show that $\algname$ belongs to $\cA$. Choosing $\bar C = R_K/d$, Theorem \ref{thm:unknown} immediately suggests that $\algname$ enjoys a $R_K$ regret in the uncorrupted case (i.e., $C=0$). Thus $\algname$ belongs to $\cA$. Then we will show that $\algname$ is the best possible one in $\cA$. On the one hand, Theorem \ref{thm:unknown} suggests that $\algname$ suffers a linear regret when $C>\bar C = R_K/d$. On the other hand, Theorem \ref{thm:lower_uc} shows that any algorithm with $R_K$ regret in the uncorrupted case should have a linear regret when $C=\Omega(R_K/d)$. These together imply that $\algname$ is optimal within $\cA$. 

\end{remark}

\section{Conclusion and Future Work}
In this work, we study corrupted linear contextual bandits. We propose a $\algname$ algorithm based on a weighted ridge regression with truncated inverse exploration bonus weights. We show that for both cases when the corruption level $C$ is known or unknown to the agent, $\algname$ achieves a regret that matches the lower bound up to logarithmic factors. We are also interested in achieving the optimal regret when specializing our algorithm to the misspecified linear contextual bandits.

\appendix

\section{Instance-dependent Regrets}\label{sec:instance}

Prior works \citep{lykouris2018stochastic, li2019stochastic,zhao2021linear} have proved instance-dependent regret bounds for corruption-robust linear bandits. We show that  
$\algname$ also enjoys an instance-dependent regret bound. Following \citet{abbasi2011improved}, we define the minimal sub-optimality gap as follows.
\begin{definition}[Minimal sub-optimality gap]
For each round $k\in [K]$ and any action $\xb\in \cD_k$, the sub-optimality gap $\Delta_{\xb,k}$ is defined as
\begin{align}   
   \Delta_{\xb,k}=\max_{\xb^* \in \cD_k}\la \btheta^*, \xb^*\ra - \la \btheta^*, \xb\ra,\notag
\end{align}
and the minimal sub-optimality gap is defined as
\begin{align}
    \Delta=\min_{k\in [K], \xb \in \cD_k}\big\{\Delta_{\xb,k}: \Delta_{\xb,k}\ne 0\big\}.\label{definition-gap-min}
\end{align}
\end{definition}
We assume that the minimal sub-optimality gap is strictly positive.
\begin{assumption}\label{assumption:gap}
The minimal sub-optimality gap is strictly positive, i.e., $\Delta>0$.
\end{assumption}
Under the assumption of positive minimal sub-optimality, the following theorem provides an instance-dependent regret guarantee for $\algname$. 
\begin{theorem}\label{theorem:known-C-gap}
Under the same conditions of Theorem \ref{theorem:known-C}, with high probability at least $1-\delta$, the regret of Algorithm \ref{algorithm-bandit-known} in the first $K$ rounds is upper bounded by
\begin{align*}
   \text{Regret}(K)&\leq O\bigg(R^2d^2\log^2\big((1+KL^2/\lambda)/\delta\big)/\Delta+\frac{\alpha^2dC^2}{\Delta} \times \sqrt{\log\big(3+C^2L^2K/({R^2\lambda \delta})\big)}\notag\\
      &\qquad +S^2d\lambda \log(1+KL^2/\lambda)/\Delta +\frac{Rd^{1.5}}{\alpha}\times \sqrt{\log^{3}\big((1+KL^2/\lambda)/\delta\big)}\notag\\
    &\qquad +\frac{dS\sqrt{\lambda}}{\alpha}\times\sqrt{\log^{2}\big((1+KL^2/\lambda)/\delta\big)}+d C \sqrt{\log^{2}\big((1+KL^2/\lambda)/\delta\big)}\bigg)
\end{align*}
\CC{In addition, if choosing $\alpha=(R\sqrt{d}+\sqrt{\lambda}S)/C$ and $\lambda=R^2/S^2$, the regret can be upper bounded by}
\begin{align*}
    \text{Regret}(K)&\leq  \tilde O(d^2/\Delta + dC).
\end{align*}
\end{theorem}
\begin{remark}
Our regret is strictly better than the $\tilde O(d^{2.5}C/\Delta + d^6/\Delta^2)$ regret proved by \citet{li2019stochastic} under a stronger assumption. \CC{Meanwhile, \citet{zhao2021linear} implies an $\tilde O(d^2C/\Delta)$ regret for their algorithm under the known $C$ case, which is also worse than our result.}
\end{remark}

\section{Overview of Key Proof Techniques}\label{section:6}
In this section, we give an overview of the main technical difficulty and our proof technique to derive Theorem \ref{theorem:known-C}. 

By the standard regret decomposition technique from \citet{abbasi2011improved}, we upper bound the regret by the sum of the exploration bonuses times the confidence radius: 
\begin{align}
    \text{Regret}(K) = O\bigg(\beta \cdot \sum_{k=1}^K\sqrt{\xb_k^{\top}\bSigma_k^{-1}\xb_k} \bigg).\label{help:98}
\end{align}
Lemma \ref{lemma:weighted-concentration} suggests $\beta \sim R\sqrt{d}+\alpha C$. Therefore, we only need to bound the summation of the exploration bonuses. For the basic case when $w_k = 1$, we bound it using the elliptical potential lemma \citep{abbasi2011improved} as follows
\begin{align}
     \sum_{w_k = 1}\sqrt{\xb_k^{\top}\bSigma_k^{-1}\xb_k} \leq \sum_{k=1}^K\sqrt{\xb_k^{\top}\bigg(\lambda\Ib + \sum_{i=1}^{k-1}\xb_i\xb_i^\top\bigg)^{-1}\xb_k} \sim \tilde O(\sqrt{dK}),\label{eq:101}
\end{align}
which contributes to the corruption-independent term $dR\sqrt{K}$ in our regret. For the case when $w_k < 1$, however, we are facing the \emph{weighted} covariance matrix and cannot directly use the elliptical potential lemma. A trivial approach is to lower bound the weights by their \emph{uniform} lower bound, i.e., 
\begin{align}
    \lambda\Ib + \sum_{i=1}^{k-1}w_i\xb_i\xb_i^\top \succeq \min_{1\leq i\leq k-1}{w_i} \cdot \bigg(\lambda \Ib+\sum_{i=1}^{k-1} \xb_i\xb_i^\top\bigg).\label{help:99}
\end{align}
By the definition of the weight $w_i$ in Algorithm \ref{algorithm-bandit-known} and a crude upper bound for the exploration bonus, we conclude from the definition of $w_k$ that $w_k  = \Omega(\alpha)$. Substituting it into~\eqref{help:99}, we only obtain a regret $\tilde O(\sqrt{dK}/\alpha)$, which is not satisfying.


To overcome this issue, we recall the definition for weight $w_k<1$ in Algorithm \ref{algorithm-bandit-known}: $w_k=\alpha/ \|\xb_k\|_{{\bSigma}_k^{-1}}$ and we can bound the summation of the exploration bonuses as
\begin{align}
     \sum_{w_k <1}\sqrt{\xb_k^{\top}\bSigma_k^{-1}\xb_k} =  \sum_{k=1}^K w_k\xb_k^{\top}\bSigma_k^{-1}\xb_k/\alpha \sim \tilde O(d/\alpha).\label{help:100}
\end{align}
Combining the results in \eqref{eq:101} and \eqref{help:100} into \eqref{help:98}, we can prove the final regret.

\section{Proof of Theorem \ref{theorem:known-C}}\label{section: proof of known C}
\CC{In this section, we provide the proof of Theorem \ref{theorem:known-C}. For simplicity, we use $\cE$ to denote the following event:
 \begin{align}
     \cE=\bigg\{\|\btheta_k-\btheta^*\|_{\bSigma_k}\leq \beta,\forall k\in [K]\bigg\}.\notag
 \end{align}
Lemma \ref{lemma:weighted-concentration} shows that $ \Pr(\cE)\ge 1-\delta$.}
\begin{lemma}\label{lemma:one-step-regret}
If setting the confidence radius $\beta=R\sqrt{d\log \big((1+KL^2/\lambda)/\delta\big)}+\alpha C+\sqrt{\lambda}S$ in Algorithm~\ref{algorithm-bandit-known}, then on the event $\cE$, for each round $k\in[K]$, the regret at round $k$ is upper bounded by
\begin{align*}
    \Delta_k= \max_{\xb \in \cD_k}\la \btheta^*, \xb\ra - \la \btheta^*, \xb_k\ra\leq 2\beta\sqrt{\xb_k^{\top}\bSigma_k^{-1}\xb_k}.
\end{align*}
\end{lemma}
\begin{proof}[Proof of Theorem \ref{theorem:known-C}]
Based on the event $\cE$, the regret in the first $K$ round can be decomposed into two parts based on the weight $w_k$:
\begin{align}
    \text{Regret}(K)&=\sum_{k=1}^K \max_{\xb \in \cD_k}\la \btheta^*, \xb\ra - \la \btheta^*, \xb_k\ra\notag\\
    &\leq \min\bigg(2,\sum_{k=1}^K 2\beta\sqrt{\xb_k^{\top}\bSigma_k^{-1}\xb_k}\bigg)\notag\\
    &= \underbrace{\sum_{k: w_k=1}  \min\bigg(2,2\beta\sqrt{\xb_k^{\top}\bSigma_k^{-1}\xb_k}\bigg)}_{I_1}+\underbrace{\sum_{k: w_k<1}  \min\bigg(2,2\beta\sqrt{\xb_k^{\top}\bSigma_k^{-1}\xb_k}\bigg)}_{I_2},\label{eq:3-1}
\end{align}
where the inequality holds due to the Lemma \ref{lemma:one-step-regret} with the fact that the suboptimality in each round $k$ is no more than $2$. 

For the term $I_1$, we consider for all rounds $k\in [K]$ with $w_k=1$ and we assume these rounds can be listed as $\{k_1,..,k_m\}$ for simplicity. With this notation, for each $i\leq m$, we can construct the auxiliary covariance matrix $\Ab_i=\lambda \Ib+ \sum_{j=1}^{i-1}\xb_{k_j}\xb_{k_j}^{\top}$. Due to the definition of original covariance matrix $\bSigma_k$ in Algorithm (Line \ref{algorithm1-line4}), we have
\begin{align}
    \bSigma_{k_i} \ge \lambda \Ib+ \sum_{j=1}^{i-1}w_{k_j}\xb_{k_j}\xb_{k_j}^{\top} = \Ab_i. \notag
\end{align} 
According to Lemma \ref{lemma:inverse}, it further implies that for vector $\xb_{k_i}$, we have
\begin{align}
    \xb_{k_i}^{\top}\bSigma_{k_i}^{-1} \xb_{k_i}\leq \xb_{k_i}^{\top}\Ab_i^{-1} \xb_{k_i}.\label{eq:3-2}
\end{align}
Therefore, the term $I_1$ can be bounded by
\begin{align}
    I_1&= \sum_{k: w_k=1}  \min\bigg(2, 2\beta \sqrt{\xb_k^{\top}\bSigma_{k}^{-1} \xb_k}\bigg)\notag\\
    &\leq \sum_{i=1}^m 2\beta  \min\bigg(1,\sqrt{\xb_{k_i}^{\top}\bSigma_{k_i}^{-1} \xb_{k_i}} \bigg)\notag\\
    &\leq 2\beta\sum_{i=1}^m \min\bigg(1,\sqrt{\xb_{k_i}^{\top}\Ab_i^{-1} \xb_{k_i}}\bigg)\notag\\
    &\leq 2\beta \sqrt{\sum_{i=1}^m 1 \times \sum_{i=1}^m \min\Big(1,\xb_{k_i}^{\top}\Ab_i^{-1} \xb_{k_i}\Big)}\notag\\
    &\leq 2\beta\sqrt{2dK\log(1+KL^2/\lambda)},\label{eq:3-3}
\end{align}
where the first inequality holds since $\beta\ge 1$, the second inequality holds due to \eqref{eq:3-2}, the third inequality holds due to Cauchy-Schwarz inequality, the last inequality holds due to Lemma \ref{Lemma:abba} with the facts that $m\leq K$ and $\|\xb_{k_i}\|_2\leq L$.

For the second term $I_2$, according to the definition for weight $w_k<1$ in Algorithm \ref{algorithm-bandit-known}, we have $w_k=\alpha /\sqrt{\xb_k^{\top}\bSigma_{k}^{-1} \xb_k}$, which implies that 
\begin{align}
    I_2&=\sum_{k: w_k<1}  \min\bigg(2,2\beta\sqrt{\xb_k^{\top}\bSigma_k^{-1}\xb_k}\bigg)\notag\\
    &= \sum_{k: w_k<1}  \min\Big(2,2\beta w_k\xb_k^{\top}\bSigma_k^{-1}\xb_k/\alpha\Big)\notag\\
    &\leq \sum_{k: w_k<1}\min \Big((2+2\beta/\alpha), (2+2\beta/\alpha) w_k\xb_k^{\top}\bSigma_k^{-1}\xb_k\Big)\notag\\
    &= \sum_{k: w_k<1}(2+2\beta/\alpha) \min \Big(1, w_k\xb_k^{\top}\bSigma_k^{-1}\xb_k\Big),\label{eq:add3-4}
\end{align}
where the second equation holds due to the definition of weight $w_k$.
Now, we assume the rounds with weight $w_k<1$ can be listed as $\{k_1,..,k_m\}$ for simplicity. In addition, we introduce the auxiliary vector $x'_i$ as $x'_i=\sqrt{w_{k_i}}\xb_{k_i}$ and matrix $\bSigma'_i$ as 
\begin{align*}
    \bSigma'_i= \lambda \Ib+\sum_{j=1}^{i-1}w_{k_j} \xb_{k_j}\xb_{k_j}^\top= \lambda \Ib+\sum_{j=1}^{i-1}\xb'_{j}(\xb'_{j})^\top.
\end{align*}
 According to Lemma \ref{lemma:inverse}, we have  $(\bSigma'_{i})^{-1}\succeq \bSigma_{k_i}^{-1}$. Therefore, for each $i \in [m]$, we have
\begin{align}
   \xb_{k_i}^{\top} (\bSigma'_{i})^{-1}\xb_{k_i}&\ge \xb_{k_i}^{\top} \bSigma_{k_i}^{-1}\xb_{k_i},\label{eq:3-5}
\end{align}
where the inequality holds due to $(\bSigma'_{i})^{-1}\succeq \bSigma_{k_i}^{-1}$.
Now, taking a summation of \eqref{eq:3-5} over all rounds $k_i$, we have 
\begin{align}
    \sum_{i=1}^m \min \Big(1,w_{k_i}\xb_{k_i}^{\top}\bSigma_{k_i}^{-1}\xb_{k_i}\Big)&\leq \sum_{i=1}^m  \min \Big(1,w_{k_i} \xb_{k_i}^{\top}(\bSigma'_{i})^{-1}\xb_{k_i}\Big)\notag\\
    &=\sum_{i=1}^m\min \Big(1, (\xb_i')^{\top}(\bSigma'_{i})^{-1}\xb'_{i}\Big)\notag\\
    &\leq 2d \log (1+KL^2/\lambda),\label{eq:3-6}
\end{align}
  where the first inequality holds due to \eqref{eq:3-5}, the second inequality holds due to Lemma \ref{Lemma:abba} with the facts that $m\leq K$. Substituting the result in \eqref{eq:3-6} into \eqref{eq:add3-4}, the term $I_2$ can be upper bounded by
  \begin{align}
      I_2&\leq  \sum_{k: w_k<1}(2+2\beta/\alpha) \min \Big(1, w_k\xb_k^{\top}\bSigma_k^{-1}\xb_k\Big)\notag\\
      &\leq (2+2\beta/\alpha)\times 2d \log (1+KL^2/\lambda).\label{eq:3-7}
  \end{align}
  
 Finally, substituting the results in \eqref{eq:3-3} and \eqref{eq:3-7} into \eqref{eq:3-1}, the regret can be upper bounded by
 \begin{align*}
      \text{Regret}(K)
      &\leq  2\beta\sqrt{2dK\log(1+KL^2/\lambda)}+(2+2\beta/\alpha)\times 2d \log (1+KL^2/\lambda)\notag\\
      &=O\bigg(dR\sqrt{K\log^2\big((1+KL^2/\lambda)/\delta\big)}+\alpha C\sqrt{dK\log^2\big((1+KL^2/\lambda)/\delta\big)}\notag\\
      &\qquad +S\sqrt{d\lambda K \log(1+KL^2/\lambda)} + \frac{Rd^{1.5}}{\alpha}\times \sqrt{\log^{3}\big((1+KL^2/\lambda)/\delta\big)}\notag\\
      &\qquad + \frac{dS\sqrt{\lambda}}{\alpha}\times\sqrt{\log^{2}\big((1+KL^2/\lambda)/\delta\big)} + d C \sqrt{\log^{2}\big((1+KL^2/\lambda)/\delta\big)} \bigg).
 \end{align*}
  Therefore, we complete the proof of Theorem \ref{theorem:known-C}.
\end{proof}
\section{Proof of Theorem \ref{theorem:known-C-gap}}
In this section, we present the detailed proof of  Theorem \ref{theorem:known-C}.
\begin{proof}[Proof of Theorem \ref{theorem:known-C-gap}]
Based on the event $\cE$, the regret in round $k\in[K]$ is upper bounded by
\begin{align*}
 \Delta_k=\max_{\xb \in \cD_k}\la \btheta^*, \xb\ra - \la \btheta^*, \xb_k\ra\leq 2\beta\sqrt{\xb_k^{\top}\bSigma_k^{-1}\xb_k}.
\end{align*}
On the other hand, according to Assumption \ref{assumption:gap}, the regret in round $k\in [K]$ satisfies that $\Delta_k=0$ or $\Delta_k\ge \Delta$. Combining these two results, for round $k\in [K]$ with uncertainty $2\beta\sqrt{\xb_k^{\top}\bSigma_k^{-1}\xb_k}< \Delta$, the regret must satisfy $\Delta_k=0$.
Therefore, the regret in the first $K$ rounds can be decomposed to two part based on the weight $w_k$ and exploration bonus $\sqrt{\xb_k^{\top}\bSigma_k^{-1}\xb_k}$:
\begin{align}
    \text{Regret}(K)&=\sum_{k=1}^K \max_{\xb \in \cD_k}\la \btheta^*, \xb\ra - \la \btheta^*, \xb_k\ra\notag\\
    &=\sum_{k: 2\beta\sqrt{\xb_k^{\top}\bSigma_k^{-1}\xb_k}\ge \Delta} \max_{\xb \in \cD_k}\la \btheta^*, \xb\ra - \la \btheta^*, \xb_k\ra \notag\\
    &\leq \min\bigg(2,\sum_{k: 2\beta\sqrt{\xb_k^{\top}\bSigma_k^{-1}\xb_k}\ge \Delta} 2\beta\sqrt{\xb_k^{\top}\bSigma_k^{-1}\xb_k}\bigg)\notag\\
    &=\sum_{k: w_k=1, 2\beta\sqrt{\xb_k^{\top}\bSigma_k^{-1}\xb_k}\ge \Delta}  \min\bigg(2,2\beta\sqrt{\xb_k^{\top}\bSigma_k^{-1}\xb_k}\bigg)\notag\\
&\qquad +\sum_{k: w_k<1,     2\beta\sqrt{\xb_k^{\top}\bSigma_k^{-1}\xb_k}\ge \Delta}  \min\bigg(2,2\beta\sqrt{\xb_k^{\top}\bSigma_k^{-1}\xb_k}\bigg)\notag\\
    &\leq \underbrace{\sum_{k: w_k=1, k: 2\beta\sqrt{\xb_k^{\top}\bSigma_k^{-1}\xb_k}\ge \Delta }  \min\bigg(2,2\beta\sqrt{\xb_k^{\top}\bSigma_k^{-1}\xb_k}\bigg)}_{J_1}\notag\\
    &\qquad +\underbrace{\sum_{k: w_k<1}  \min\bigg(2,2\beta\sqrt{\xb_k^{\top}\bSigma_k^{-1}\xb_k}\bigg)}_{J_2},\label{eq:5-1}
\end{align}
where the inequality holds due to Lemma \ref{lemma:one-step-regret} with the fact that the suboptimality in each round is no more than $2$. 
Notice that the term $J_2$ is equal to the term $I_2$ in the proof of Theorem \ref{theorem:known-C} (See \eqref{eq:3-1}) and with the same argument, it can be upper bounded by
\begin{align}
    J_2&\leq O\bigg(\frac{Rd^{1.5}}{\alpha}\times \sqrt{\log^{3}\big((1+KL^2/\lambda)/\delta\big)}+\frac{dS\sqrt{\lambda}}{\alpha}\times\sqrt{\log^{2}\big((1+KL^2/\lambda)/\delta\big)}\notag\\
    &\qquad+ d C \sqrt{\log^{2}\big((1+KL^2/\lambda)/\delta\big)}\bigg),\label{eq:5-2}
\end{align}
where the inequality comes from \eqref{eq:3-7}.
For the term $J_1$, we consider for all rounds $k\in [K]$ with $w_k=1$ and exploration bonus $2\beta\sqrt{\xb_k^{\top}\bSigma_k^{-1}\xb_k}\ge \Delta$. For simplicity, we assume these rounds can be listed as $\{k_1,..,k_m\}$.
 With this notation, for each $i\leq m$, we can construct the auxiliary covariance matrix $\Ab_i=\lambda \Ib+ \sum_{j=1}^{i-1}\xb_{k_j}\xb_{k_j}^{\top}$. Due to the definition of original covariance matrix $\bSigma_k$ in Algorithm (Line \ref{algorithm1-line4}), we have
\begin{align}
    \bSigma_{k_i} \ge \lambda \Ib+ \sum_{j=1}^{i-1}w_{k_j}\xb_{k_j}\xb_{k_j}^{\top} = \Ab_i. \notag
\end{align}
According to Lemma \ref{lemma:inverse}, it further implies that for vector $\xb_{k_i}$, we have
\begin{align}
    \xb_{k_i}^{\top}\bSigma_{k_i}^{-1} \xb_{k_i}\leq \xb_{k_i}^{\top}\Ab_i^{-1} \xb_{k_i}.\label{eq:5-3}
\end{align}
Therefore, the term $J_1$ can be bounded by
\begin{align}
    J_1&= \sum_{k: w_k=1, 2\beta \sqrt{\xb_k^{\top}\bSigma_{k}^{-1} \xb_k} \ge \Delta}  \min\bigg(2, 2\beta \sqrt{\xb_k^{\top}\bSigma_{k}^{-1} \xb_k}\bigg)\notag\\
    &\leq \sum_{i=1}^m 2\beta  \min\bigg(1,\sqrt{\xb_{k_i}^{\top}\bSigma_{k_i}^{-1} \xb_{k_i}} \bigg)\notag\\
    &\leq 2\beta\sum_{i=1}^m \min\bigg(1,\sqrt{\xb_{k_i}^{\top}\Ab_i^{-1} \xb_{k_i}}\bigg)\notag\\
    &\leq 2\beta \sqrt{\sum_{i=1}^m 1 \times \sum_{i=1}^m \min\Big(1,\xb_{k_i}^{\top}\Ab_i^{-1} \xb_{k_i}\Big)}\notag\\
    &\leq 2\beta\sqrt{2dm\log(1+KL^2/\lambda)},\label{eq:5-4}
\end{align}
where the first inequality holds since $\beta\ge 1$, the second inequality holds due to \eqref{eq:5-3}, the third inequality holds due to Cauchy-Schwarz inequality, the fourth inequality holds due to Lemma \ref{Lemma:abba} with the facts that $m\leq K$ and $\|\xb_{k_i}\|_2\leq L$. On the other hand, the term $J_1$ is lower bounded by
\begin{align}
    J_1= \sum_{k: w_k=1, 2\beta \sqrt{\xb_k^{\top}\bSigma_{k}^{-1} \xb_k}\ge \Delta}  \min\bigg(2, 2\beta \sqrt{\xb_k^{\top}\bSigma_{k}^{-1} \xb_k}\bigg) \ge m\times \Delta,\label{eq:5-5}
\end{align}
where the inequality holds due to the definition of $k_i$ with the fact that $\Delta\leq 2$. Combining the upper and lower bound for term $J_1$, we have
\begin{align*}
    m\times \Delta \leq 2\beta\sqrt{2dm\log(1+KL^2/\lambda)},
\end{align*}
which further implies that
\begin{align}
    m\leq O\big(\beta^2d \log(1+KL^2/\lambda)/\text{gap}^2_{\min} \big).\label{eq:5-6}
\end{align}
Substituting the upper bound of $m$ in \eqref{eq:5-6} into \eqref{eq:5-4}, the term $J_1$ can be upper bounded by
\begin{align}
    J_1\leq O\big(\beta^2 d\log(1+KL^2/\lambda)/\Delta\big).\label{eq:5-7}
\end{align}
Finally, substituting the upper bounds of term $J_2$ in \eqref{eq:5-2} and term $J_1$ in \eqref{eq:5-7} into \eqref{eq:5-1}, the regret can be upper bounded by
\begin{align*}
    \text{Regret}(K)&\leq O\big(\beta^2 d\log(1+KL^2/\lambda)/\Delta\big)+O\bigg(\frac{Rd^{1.5}}{\alpha}\times \sqrt{\log^{3}\big((1+KL^2/\lambda)/\delta\big)}\notag\\
    &\qquad +\frac{dS\sqrt{\lambda}}{\alpha}\times\sqrt{\log^{2}\big((1+KL^2/\lambda)/\delta\big)}+d C \sqrt{\log^{2}\big((1+KL^2/\lambda)/\delta\big)}\bigg)\notag\\
      &=O\bigg(R^2d^2\log^2\big((1+KL^2/\lambda)/\delta\big)/\Delta+\frac{\alpha^2dC^2}{\Delta}\times \sqrt{\log\big(3+C^2L^2K/({R^2\lambda \delta})\big)}\notag\\
      &\qquad +S^2d\lambda \log(1+KL^2/\lambda)/\Delta +\frac{Rd^{1.5}}{\alpha}\times \sqrt{\log^{3}\big((1+KL^2/\lambda)/\delta\big)}\notag\\
    &\qquad +\frac{dS\sqrt{\lambda}}{\alpha}\times\sqrt{\log^{2}\big((1+KL^2/\lambda)/\delta\big)}+d C \sqrt{\log^{2}\big((1+KL^2/\lambda)/\delta\big)}\bigg).
\end{align*}
Therefore, we complete the proof of Theorem \ref{theorem:known-C-gap}.
\end{proof}

\section{Proof of Theorem \ref{thm:unknown}}
\begin{proof}[Proof of Theorem \ref{thm:unknown}]
We discuss two cases here. 
 \begin{itemize}[leftmargin = *]
    \item For the case $C \leq \bar{C}$, we know that $\bar{C}$ is still a valid upper bound of the corruption level. Thus, $\algname$ with a $\bar{C}$ corruption level runs successfully, and its regret is upper bounded by $\tilde O(dR\sqrt{K} + d\bar{C}) = \tilde O(dR\sqrt{K}+d\bar{C})$ as Theorem \ref{theorem:known-C} suggests. 
    \item For the case $C = \Omega(\bar{C})$, $\algname$ can not guarantee a sublinear regret. Thus a trivial regret bound (i.e., regret at each round is bounded by 2) applies. 
\end{itemize}
\end{proof}

\section{Proof of Theorem \ref{thm:lower_uc}}
We introduce our proof of Theorem \ref{thm:lower_uc}, which is adapted from \citet{bogunovic2021stochastic}. 
\begin{proof}[Proof of Theorem \ref{thm:lower_uc}]
\CC{In this proof, we consider an arbitrary algorithm satisfying the conditions in the statement of Theorem \ref{thm:lower_uc}, which will run $K$ rounds for any bandit instance. We consider an uncorrupted bandit instance $A_0$ defined as follows. $A_0$ has the decision sets $\cD_k = \cD$. Here $\cD = \{\ab_i\}_{1 \leq i \leq d}$, where $\ab_i = \eb_i$ is the basis in the $d$-dimensional space. Let $\btheta_0^* = (1/4, \underbrace{1/8, \dots, 1/8}_{(d-1)-\text{times}}) \in \RR^d$ and $\epsilon_i = 0$. It is easy to see that the optimal policy is to select $\ab_1$ at each round, and the regret to select a sub-optimal arm is $1/8$. Since the regret of the algorithm without corruption satisfies $\EE\big[\text{Regret}(K)\big]<R_
K$, and all the regret comes from selecting $\ab_2, \dots, \ab_d$, we have the expected number of rounds to select $\ab_2, \dots, \ab_d$ is at most $R_K/(1/8) = 8R_K$. Then by the pigeonhole principle, there exists some $2 \leq i \leq d$ such that the expected number of times to select $\ab_i$ is less than $8R_K/(d-1)$. Without loss of generality, we suppose $i = 2$. Then by Markov inequality, with probability at least $1/2$, the number of times to select $\ab_2$ is less than $16R_K/(d-1)$.}

\CC{Next, we consider a corrupted bandit instance $A_1$ defined as follows. $A_1$ has the same decision set $\cD = \{\eb_i\}$ as $A_0$, while it has a different $\btheta_1^* = (1/4, 3/8,\underbrace{ 1/8, \dots, 1/8}_{(d-2)-\text{times}})$. $A_1$ is also noiseless, i.e., $\epsilon_i = 0$. Unlike $A_0$, we have an adversary to attack $A_1$ as follows: whenever $\ab_2$ is selected and the total corruption level up to the previous step is no more than $4R_K/(d-1)-1/4$, the adversary corrupts the reward from $3/8$ to $1/8$. Otherwise, the adversary stops to corrupt the reward. With this adversary, the  
corruption level $C$ is upper bounded by $4R_K/(d-1)-1/4+1/4=\Omega(R_K/d)$.}
 
\CC{For this adversary, since for $A_1$, each selection of $\ab_2$ returns a reward $1/8$, then the agent can not tell the difference between $A_0$ and $A_1$ until the total corruption level reaches the threshold $4R_k/(d-1)$ and the adversary stops to corrupt the reward. 
Therefore, the sequence of rounds for the agent to select $\ab_2$ with $A_1$ instance is the same as the sequence for the agent to select $\ab_2$ with $A_0$, until the number of rounds to select action $\ab_2$ reaches $4R_k/(d-1)/(1/4)=16R_k/(d-1)$. However, when the total number of times to select $\ab_2$ is less than $16R_K/(d-1)$, the agent cannot differentiate $A_0$ and $A_1$ and will follow the same action sequence as $A_0$. In this case, since for $A_1$, $\ab_2$ is the optimal action, and all the other actions suffer a $1/8$ regret, then the regret on $A_1$ is at least $1/8\cdot (K - 16R_K/(d-1)) = \Omega(K)$, where we use the fact that $R_K \leq O(K)$. Therefore, with probability at least $1/2$, the regret is at least $\Omega(K)$, which further implies that the expected regret is lower bounded by $\EE[\text{Regret(K)}]\ge 1/2\times \Omega(K) =\Omega(K)$. Thus, we finish the proof of Theorem~\ref{thm:lower_uc}.}
\end{proof}

\section{Proof of Lemmas in Sections \ref{section:5} and Appendix \ref{section: proof of known C}}
\subsection{Proof of  Lemma \ref{lemma:weighted-concentration}}
\begin{proof}[Proof of Lemma \ref{lemma:weighted-concentration}]
According to the definition of estimated vector $\btheta_k$ in Algorithm \ref{algorithm-bandit-known} (Line \ref{algorithm1-line5}), we have
\begin{align}
    \btheta_k&=\bSigma_{k}^{-1}\bbb_{k}=\bSigma_k^{-1} \sum_{i=1}^{k-1}w_i \xb_i r_i=\bSigma_k^{-1}\sum_{i=1}^{k-1}w_i\xb_i (\xb_i^{\top} \btheta +\eta_i+c_i). \notag
\end{align}
This equation further implies that the difference between estimated vector $\btheta_k$ and the unknown vector $\btheta^*$ can be decomposed as: 
\begin{align}
    \|\btheta_k-\btheta^*\|_{\bSigma_k}&=\big\|\bSigma_k^{-1}\sum_{i=1}^{k-1}w_i\xb_i (\xb_i^{\top} \btheta^* +\eta_i+c_i) -\btheta^*\big\|_{\bSigma_k}\notag\\
    &=\bigg\|\bSigma_k^{-1}\sum_{i=1}^{k-1}w_i\xb_i (\xb_i^{\top} \btheta +\eta_i +c_i)-\bSigma_k^{-1} \Big(\sum_{i=1}^{k-1}w_i\xb_i\xb_i^{\top} +\lambda \Ib\Big) \btheta^*\bigg\|_{\bSigma_k}\notag\\
    &=\bigg\|\bSigma_k^{-1}\sum_{i=1}^{k-1}w_i\xb_i \eta_i+\bSigma_k^{-1}\sum_{i=1}^{k-1}w_i\xb_i c_i- \lambda \bSigma_k^{-1}\btheta^* \bigg\|_{\bSigma_k}\notag\\
    &\leq \underbrace{\bigg\|\bSigma_k^{-1}\sum_{i=1}^{k-1}w_i\xb_i \eta_i \bigg\|_{\bSigma_k}}_{\text{Stochastic error}:I_1}+\underbrace{\bigg\|\bSigma_k^{-1}\sum_{i=1}^{k-1}w_i\xb_i c_i \bigg\|_{\bSigma_k}}_{\text{Corruption  error}:I_2}+\underbrace{\bigg\|\lambda \bSigma_k^{-1}\btheta^* \bigg\|_{\bSigma_k}}_{\text{Regularization error}:I_3},\label{eq:2-1}
\end{align}
where the inequality holds due to the fact that $\|\ab+\bbb+\bc \|_{\bSigma_k}\leq \|\ab\|_{\bSigma_k}+\|\bbb\|_{\bSigma_k}+\|\bc\|_{\bSigma_k}$.

For the stochastic error term $I_1$, it can be bounded by the concentration Lemma \ref{lemma: concentration-OFUL} in \citet{abbasi2011improved}. More specifically, we introduce the auxiliary vector $\xb'_i$ and noise $\eta'_i$ such that $\xb'_i=\sqrt{w_i}\xb_i$ and $\eta'_i=\sqrt{w_i} \eta_i$. 
According to the definition of weight $\btheta_i$ in Algorithm 
 (Line \ref{algorithm1-line6}), both of these two situations satisfies that the weight $\btheta_i$ is bounded by $w_i\leq 1$. Since the original vector $\xb_i$ satisfies that $\|\xb_i\|_2\leq L$ and the original stochastic noise $\eta_i$ is $R$-sub Gaussian, these results further imply that 
 \begin{align}
     \|\xb'_i\|_2=\|\sqrt{w_i}\xb_i\|_2\leq L, \eta'_i=\sqrt{w_i}\eta_i \text{ is $R$-sub Gaussian.} \notag
 \end{align}
  With this notation, the covariance matrix $\bSigma_k$ and the stochastic error term $I_1$ can be rewritten and bounded as:
  \begin{align}
     \bSigma_k&=\lambda \Ib+ \sum_{i=1}^{k-1}w_i\xb_i\xb_i^{\top} =\lambda \Ib+ \sum_{i=1}^{k-1}\xb'_i(\xb'_i)^{\top} \notag
\end{align}
\begin{align}
     I_1&=\bigg\|\bSigma_k^{-1}\sum_{i=1}^{k-1}w_i\xb_i \eta_i \bigg\|_{\bSigma_k}\notag\\
     &=\bigg\|\sum_{i=1}^{k-1}w_i\xb_i \eta_i\bigg\|_{\bSigma_k^{-1}}\notag\\
     &=\bigg\|\sum_{i=1}^{k-1}\xb'_i \eta'_i\bigg\|_{\bSigma_k^{-1}}\notag\\
     &\leq \sqrt{2R^2  \log\bigg(\frac{\det(\bSigma_k)^{1/2}\det(\bSigma_1)^{-1/2}}{\delta}\bigg)}\notag\\
     &\leq R\sqrt{d\log \big((1+KL^2/\lambda)/\delta\big)},\label{eq:I_1}
  \end{align}
  where the first inequality holds due to Lemma \ref{lemma: concentration-OFUL} and the second inequality holds due to the facts that $\bSigma_k=\lambda \Ib_+ \sum_{i=1}^{k-1}\xb'_i(\xb'_i)^{\top}$ and $\|\xb'\|_2\leq L$.

For the corruption error term $I_2$, it can be bounded by
\begin{align}
     I_2&=\bigg\|\bSigma_k^{-1}\sum_{i=1}^{k-1}w_i\xb_i c_i \bigg\|_{\bSigma_k}\notag\\
     &=\bigg\|\bSigma_k^{-1/2}\sum_{i=1}^{k-1}w_i\xb_i  c_i\bigg\|_2\notag\\
     &\leq \sum_{i=1}^{k-1}\bigg\|\bSigma_k^{-1/2}w_i  \xb_i  c_i\bigg\|_2\notag\\
     &= \sum_{i=1}^{k-1} |c_i|\times w_i \|\bSigma_k^{-1/2}\xb_i\|\notag\\
     &\leq \sum_{i=1}^{k-1}|c_i|\alpha\notag\\
     &\leq \alpha C ,\label{eq:I_2}
\end{align}
where the first inequality holds due to the fact that $\|\ab+\bbb\|_2\leq \|\ab\|_2+\|\bbb\|_2$, the second inequality holds due to the definition of weight $w_i$ in Algorithm (Line \ref{algorithm1-line6}) with the fact that $\bSigma_k \succeq \bSigma_i$ and the last inequality holds due to the definition of corruption level $C$.

For the regularization error term $I_3$, we have 
\begin{align}
    I_3&=\big\|\lambda \bSigma_k^{-1}\btheta^* \big\|_{\bSigma_k}=\lambda \big\|\btheta^{*}\big\|_{\bSigma_k^{-1}}\leq\sqrt{\lambda}  \|\btheta^*\|_2\leq \sqrt{\lambda} S,\label{eq:I_3}
\end{align}
where the first inequality holds due to $\big\|\btheta^{*}\big\|_{\bSigma_k}\leq \|\btheta^*\|_2/\sqrt{\lambda_{\min}(\bSigma_k)}$ with the fact that $\bSigma_k=\lambda \Ib+\sum_{i=1}^{k-1}w_i\xb_i \xb_i^{\top}\succeq \lambda \Ib$ and the last inequality holds due to the assumption that $\|\btheta^*\|_2\leq S$.

Finally, substituting the results in \eqref{eq:I_1}, \eqref{eq:I_2} and \eqref{eq:I_3} into \eqref{eq:2-1}, we have
\begin{align*}
    \|\btheta_k-\btheta^*\|_{\bSigma_k}&\leq I_1+I_2+I_3\leq R\sqrt{d\log \big((1+KL^2/\lambda)/\delta\big)}+\alpha C +\sqrt{\lambda} S.
\end{align*}
Therefore, we finish the proof of Lemma \ref{lemma:weighted-concentration}.
\end{proof}

\subsection{Proof of Lemma \ref{lemma:one-step-regret}}
\begin{proof}[Proof of Lemma \ref{lemma:one-step-regret}]
Firstly, on the event $\cE$, for each round $k\in [K]$ and each action $\xb \in \cD_k$, we have 
\begin{align}
    \btheta_{k}^{\top}\xb+\beta\sqrt{\xb^{\top}\bSigma_{k}^{-1} \xb}- (\btheta^*)^{\top}\xb&=(\btheta_{k}-\btheta^*)^{\top}\xb+\beta\sqrt{\xb^{\top}\bSigma_{k}^{-1} \xb}\notag\\
    &\ge -\|\btheta_{k}-\btheta^*\|_{\bSigma_k} \times \|\xb\|_{\bSigma_k^{-1}}+\beta\sqrt{\xb^{\top}\bSigma_{k}^{-1} \xb}\notag\\
    &\ge -\beta \|\xb\|_{\bSigma_k^{-1}}+\beta\sqrt{\xb^{\top}\bSigma_{k}^{-1}\xb}\notag\\
    &=0,\label{eq:2-2}
\end{align}
where the first inequality holds due to the Cauchy-Schwarz inequality and the last inequality holds due to the definition of $\cE$ in Lemma \ref{lemma:weighted-concentration}. \eqref{eq:2-2} shows that our estimator in Algorithm \ref{algorithm-bandit-known} is optimistic for each action $\xb \in \cD_k$. For simplicity, we denote the optimal action at round $k$ as $\xb^*=\arg \max_{\xb \in \cD_k}(\btheta^*)^{\top}\xb$ and \eqref{eq:2-2} further implies that the regret at round $k$ can be upper bounded by
\begin{align}
    \Delta_k&=(\btheta^*)^{\top}\xb^*-(\btheta^*)^{\top}\xb_k\notag\\
    &\leq \btheta_{k}^{\top}\xb^*+\beta\sqrt{(\xb^*)^{\top}\bSigma_{k}^{-1} \xb^*}-(\btheta^*)^{\top}\xb_k\notag\\
    & \leq \btheta_{k}^{\top}\xb_k+\beta\sqrt{\xb_k^{\top}\bSigma_{k}^{-1} \xb_k}-(\btheta^*)^{\top}\xb_k\notag\\
    &= (\btheta_{k}-\btheta^*)^{\top}\xb_k+\beta\sqrt{\xb_k^{\top}\bSigma_{k}^{-1} \xb_k}\notag\\
    &\leq \|\btheta_{k}-\btheta^*\|_{\bSigma_k} \times \|\xb_{k}\|_{\bSigma_k^{-1}}+\beta\sqrt{\xb_k^{\top}\bSigma_{k}^{-1} \xb_k}\notag\\
    &\leq 2 \beta \sqrt{\xb_k^{\top}\bSigma_{k}^{-1} \xb_k},\notag
\end{align}
where the first inequality holds due to \eqref{eq:2-2}, the second inequality holds due to the selection rule in Algorithm (Line \ref{algorithm1-line8}), the third inequality holds due to the Cauchy-Schwarz inequality and the last inequality holds due to the definition of $\cE$ in Lemma \ref{lemma:weighted-concentration}. Thus, we finish the proof of Lemma \ref{lemma:one-step-regret}.
\end{proof}

\section{Auxiliary Lemmas}
\begin{lemma}[Azuma–Hoeffding inequality, \citealt{cesa2006prediction}]\label{lemma:azuma}
Let $\{\eta_k\}_{k=1}^K$ be a martingale difference sequence with respect to a filtration $\{\cG_{k}\}$ satisfying $|\eta_k| \leq R$ for some constant $R$, $\eta_k$ is $\cG_{k+1}$-measurable, $\EE\big[\eta_k|\cG_k\big] = 0$. Then for any $0<\delta<1$, with high probability at least $1-\delta$, we have 
\begin{align}
    \sum_{k=1}^K \eta_k\leq R\sqrt{2K \log (1/\delta)}.\notag
\end{align} 
\end{lemma}

\begin{lemma}[Lemma 9 in \citealt{abbasi2011improved}]\label{lemma: concentration-OFUL} Let $\{\epsilon_k\}_{k=1}^{K}$ be a real-valued stochastic process with corresponding filtration $ \{\mathcal{F}_k\}_{k=0}^{K}$ such that $\epsilon_k$ is $\mathcal{F}_k$-measure and $\epsilon_k$ is conditionally
$R$-sub-Gaussian, $i.e.$
\begin{align}
    \forall \lambda \in \RR, \EE\big[e^{\lambda\epsilon_k}|\mathcal{F}_{k-1}\big]\leq \exp\bigg(\frac{\lambda^2R^2}{2}\bigg)\notag.
\end{align}
Let $\{\xb_k\}_{k=1}^{K}$ be an $\RR^d$-valued stochastic process where $\xb_k$ is $\mathcal{F}_{k-1}$-measurable and for any $k \in [K]$, we further define $\bSigma_k=\lambda \Ib+\sum_{i=1}^{k}\xb_i\xb_i^{\top}$.
Then with probability at least $1-\delta$, for all $k\in[K]$, we have
\begin{align*}
    \bigg\|\sum_{i=1}^k\xb_i \eta_i\bigg\|_{\bSigma_k^{-1}}^2\leq 2R^2  \log\bigg(\frac{\det(\bSigma_k)^{1/2}\det(\bSigma_0)^{-1/2}}{\delta}\bigg).
\end{align*}
\end{lemma}

\begin{lemma}[Lemma 11 in \citealt{abbasi2011improved}]\label{Lemma:abba}
Let $\{\xb_k\}_{k=1}^{K}$ be a sequence of vectors in $\RR^d$, matrix $\bSigma_0$ a $d \times d$ positive definite matrix and define $\bSigma_k=\bSigma_0+\sum_{i=1}^{k} \xb_i\xb_i^{\top}$, then we have
\begin{align}
    \sum_{i=1}^{k} \min\Big\{1,\xb_i^{\top} \bSigma_{i-1}^{-1} \xb_i\Big\}\leq 2 \log \bigg(\frac{\det{\bSigma_k}}{\det{\bSigma_0}}\bigg).\notag
\end{align}
In addition, if $\|\xb_i\|_2\leq L$ holds for all $i\in [K]$, then 
\begin{align}
    \sum_{i=1}^{k} \min\Big\{1,\xb_i^{\top} \bSigma_{i-1}^{-1} \xb_i\Big\}\leq 2 \log \bigg(\frac{\det{\bSigma_k}}{\det{\bSigma_0}}\bigg)\leq 2\Big(d\log\big((\text{trace}(\bSigma_0)+kL^2)/d\big)-\log \det \bSigma_0\Big).\notag
\end{align}
\end{lemma}
\begin{lemma}[Corollary 7.7.4. (a) in \citealt{horn2012matrix}]\label{lemma:inverse}
Let $\Ab,\Bb$ be a Hermitian matrix in $\RR^{d\times d}$ and suppose $\Ab,\Bb \succ \zero$, then $\Ab \succeq \Bb$ if and only if $\Bb^{-1} \succeq \Ab^{-1}$.
\end{lemma}

\bibliographystyle{ims}
\bibliography{reference}

\begin{thebibliography}{30}
\expandafter\ifx\csname natexlab\endcsname\relax\def\natexlab#1{#1}\fi
\expandafter\ifx\csname url\endcsname\relax
  \def\url#1{\texttt{#1}}\fi
\expandafter\ifx\csname urlprefix\endcsname\relax\def\urlprefix{URL }\fi

\bibitem[{Abbasi-Yadkori et~al.(2011)Abbasi-Yadkori, P{\'a}l and
  Szepesv{\'a}ri}]{abbasi2011improved}
\textsc{Abbasi-Yadkori, Y.}, \textsc{P{\'a}l, D.} and \textsc{Szepesv{\'a}ri,
  C.} (2011).
\newblock Improved algorithms for linear stochastic bandits.
\newblock In \textit{Advances in Neural Information Processing Systems}.

\bibitem[{Agarwal et~al.(2017)Agarwal, Luo, Neyshabur and
  Schapire}]{agarwal2017corralling}
\textsc{Agarwal, A.}, \textsc{Luo, H.}, \textsc{Neyshabur, B.} and
  \textsc{Schapire, R.~E.} (2017).
\newblock Corralling a band of bandit algorithms.
\newblock In \textit{Conference on Learning Theory}. PMLR.

\bibitem[{Auer et~al.(2002)Auer, Cesa-Bianchi, Freund and
  Schapire}]{auer2002nonstochastic}
\textsc{Auer, P.}, \textsc{Cesa-Bianchi, N.}, \textsc{Freund, Y.} and
  \textsc{Schapire, R.~E.} (2002).
\newblock The nonstochastic multiarmed bandit problem.
\newblock \textit{SIAM journal on computing} \textbf{32} 48--77.

\bibitem[{Auer and Chiang(2016)}]{pmlr-v49-auer16}
\textsc{Auer, P.} and \textsc{Chiang, C.-K.} (2016).
\newblock An algorithm with nearly optimal pseudo-regret for both stochastic
  and adversarial bandits.
\newblock In \textit{Conference on Learning Theory}. PMLR.

\bibitem[{Bogunovic et~al.(2021)Bogunovic, Losalka, Krause and
  Scarlett}]{bogunovic2021stochastic}
\textsc{Bogunovic, I.}, \textsc{Losalka, A.}, \textsc{Krause, A.} and
  \textsc{Scarlett, J.} (2021).
\newblock Stochastic linear bandits robust to adversarial attacks.
\newblock In \textit{International Conference on Artificial Intelligence and
  Statistics}. PMLR.

\bibitem[{Bubeck and Cesa-Bianchi(2012)}]{bubeck2012regret}
\textsc{Bubeck, S.} and \textsc{Cesa-Bianchi, N.} (2012).
\newblock Regret analysis of stochastic and nonstochastic multi-armed bandit
  problems.
\newblock \textit{arXiv preprint arXiv:1204.5721} .

\bibitem[{Bubeck and Slivkins(2012)}]{pmlr-v23-bubeck12b}
\textsc{Bubeck, S.} and \textsc{Slivkins, A.} (2012).
\newblock The best of both worlds: Stochastic and adversarial bandits.
\newblock In \textit{Conference on Learning Theory}. JMLR Workshop and
  Conference Proceedings.

\bibitem[{Cesa-Bianchi and Lugosi(2006)}]{cesa2006prediction}
\textsc{Cesa-Bianchi, N.} and \textsc{Lugosi, G.} (2006).
\newblock \textit{Prediction, learning, and games}.
\newblock Cambridge university press.

\bibitem[{Deshpande and Montanari(2012)}]{deshpande2012linear}
\textsc{Deshpande, Y.} and \textsc{Montanari, A.} (2012).
\newblock Linear bandits in high dimension and recommendation systems.
\newblock In \textit{2012 50th Annual Allerton Conference on Communication,
  Control, and Computing (Allerton)}. IEEE.

\bibitem[{Ding et~al.(2021)Ding, Hsieh and Sharpnack}]{ding2021robust}
\textsc{Ding, Q.}, \textsc{Hsieh, C.-J.} and \textsc{Sharpnack, J.} (2021).
\newblock Robust stochastic linear contextual bandits under adversarial
  attacks.
\newblock \textit{arXiv preprint arXiv:2106.02978} .

\bibitem[{Foster et~al.(2020)Foster, Gentile, Mohri and
  Zimmert}]{foster2020adapting}
\textsc{Foster, D.~J.}, \textsc{Gentile, C.}, \textsc{Mohri, M.} and
  \textsc{Zimmert, J.} (2020).
\newblock Adapting to misspecification in contextual bandits.
\newblock \textit{Advances in Neural Information Processing Systems}
  \textbf{33} 11478--11489.

\bibitem[{Ghosh et~al.(2017)Ghosh, Chowdhury and
  Gopalan}]{ghosh2017misspecified}
\textsc{Ghosh, A.}, \textsc{Chowdhury, S.~R.} and \textsc{Gopalan, A.} (2017).
\newblock Misspecified linear bandits.
\newblock In \textit{Thirty-First AAAI Conference on Artificial Intelligence}.

\bibitem[{Gupta et~al.(2019{\natexlab{a}})Gupta, Koren and
  Talwar}]{gupta2019better}
\textsc{Gupta, A.}, \textsc{Koren, T.} and \textsc{Talwar, K.}
  (2019{\natexlab{a}}).
\newblock Better algorithms for stochastic bandits with adversarial
  corruptions.
\newblock In \textit{Conference on Learning Theory}. PMLR.

\bibitem[{Gupta et~al.(2019{\natexlab{b}})Gupta, Koren and
  Talwar}]{pmlr-v99-gupta19a}
\textsc{Gupta, A.}, \textsc{Koren, T.} and \textsc{Talwar, K.}
  (2019{\natexlab{b}}).
\newblock Better algorithms for stochastic bandits with adversarial
  corruptions.
\newblock In \textit{Conference on Learning Theory}. PMLR.

\bibitem[{Horn and Johnson(2012)}]{horn2012matrix}
\textsc{Horn, R.~A.} and \textsc{Johnson, C.~R.} (2012).
\newblock \textit{Matrix analysis}.
\newblock Cambridge university press.

\bibitem[{Kannan et~al.(2018)Kannan, Morgenstern, Roth, Waggoner and
  Wu}]{kannan2018smoothed}
\textsc{Kannan, S.}, \textsc{Morgenstern, J.~H.}, \textsc{Roth, A.},
  \textsc{Waggoner, B.} and \textsc{Wu, Z.~S.} (2018).
\newblock A smoothed analysis of the greedy algorithm for the linear contextual
  bandit problem.
\newblock \textit{Advances in neural information processing systems}
  \textbf{31}.

\bibitem[{Kirschner and Krause(2018)}]{kirschner2018information}
\textsc{Kirschner, J.} and \textsc{Krause, A.} (2018).
\newblock Information directed sampling and bandits with heteroscedastic noise.
\newblock In \textit{Conference On Learning Theory}. PMLR.

\bibitem[{Krishnamurthy et~al.(2021)Krishnamurthy, Hadad and
  Athey}]{krishnamurthy2021adapting}
\textsc{Krishnamurthy, S.~K.}, \textsc{Hadad, V.} and \textsc{Athey, S.}
  (2021).
\newblock Adapting to misspecification in contextual bandits with offline
  regression oracles.
\newblock In \textit{International Conference on Machine Learning}. PMLR.

\bibitem[{Lattimore and Szepesv{\'a}ri(2018)}]{lattimore2018bandit}
\textsc{Lattimore, T.} and \textsc{Szepesv{\'a}ri, C.} (2018).
\newblock Bandit algorithms.
\newblock \textit{preprint}  28.

\bibitem[{Lattimore and Szepesvari(2019)}]{lattimore2019learning}
\textsc{Lattimore, T.} and \textsc{Szepesvari, C.} (2019).
\newblock Learning with good feature representations in bandits and in rl with
  a generative model.
\newblock \textit{arXiv preprint arXiv:1911.07676} .

\bibitem[{Lee et~al.(2021)Lee, Luo, Wei, Zhang and Zhang}]{lee2021achieving}
\textsc{Lee, C.-W.}, \textsc{Luo, H.}, \textsc{Wei, C.-Y.}, \textsc{Zhang, M.}
  and \textsc{Zhang, X.} (2021).
\newblock Achieving near instance-optimality and minimax-optimality in
  stochastic and adversarial linear bandits simultaneously.
\newblock \textit{arXiv preprint arXiv:2102.05858} .

\bibitem[{Li et~al.(2019)Li, Lou and Shan}]{li2019stochastic}
\textsc{Li, Y.}, \textsc{Lou, E.~Y.} and \textsc{Shan, L.} (2019).
\newblock Stochastic linear optimization with adversarial corruption.
\newblock \textit{arXiv preprint arXiv:1909.02109} .

\bibitem[{Lykouris et~al.(2018)Lykouris, Mirrokni and
  Paes~Leme}]{lykouris2018stochastic}
\textsc{Lykouris, T.}, \textsc{Mirrokni, V.} and \textsc{Paes~Leme, R.} (2018).
\newblock Stochastic bandits robust to adversarial corruptions.
\newblock In \textit{Proceedings of the 50th Annual ACM SIGACT Symposium on
  Theory of Computing}.

\bibitem[{Russac et~al.(2019)Russac, Vernade and
  Capp{\'e}}]{russac2019weighted}
\textsc{Russac, Y.}, \textsc{Vernade, C.} and \textsc{Capp{\'e}, O.} (2019).
\newblock Weighted linear bandits for non-stationary environments.
\newblock \textit{Advances in Neural Information Processing Systems}
  \textbf{32}.

\bibitem[{Seldin and Lugosi(2017)}]{pmlr-v65-seldin17a}
\textsc{Seldin, Y.} and \textsc{Lugosi, G.} (2017).
\newblock An improved parametrization and analysis of the exp3++ algorithm for
  stochastic and adversarial bandits.
\newblock In \textit{Conference on Learning Theory}. PMLR.

\bibitem[{Seldin and Slivkins(2014)}]{pmlr-v32-seldinb14}
\textsc{Seldin, Y.} and \textsc{Slivkins, A.} (2014).
\newblock One practical algorithm for both stochastic and adversarial bandits.
\newblock In \textit{International Conference on Machine Learning}. PMLR.

\bibitem[{Wei et~al.(2022)Wei, Dann and Zimmert}]{wei2022model}
\textsc{Wei, C.-Y.}, \textsc{Dann, C.} and \textsc{Zimmert, J.} (2022).
\newblock A model selection approach for corruption robust reinforcement
  learning.
\newblock In \textit{International Conference on Algorithmic Learning Theory}.
  PMLR.

\bibitem[{Zhao et~al.(2021)Zhao, Zhou and Gu}]{zhao2021linear}
\textsc{Zhao, H.}, \textsc{Zhou, D.} and \textsc{Gu, Q.} (2021).
\newblock Linear contextual bandits with adversarial corruptions.
\newblock \textit{arXiv preprint arXiv:2110.12615} .

\bibitem[{Zhou et~al.(2021)Zhou, Gu and Szepesvari}]{zhou2021nearly}
\textsc{Zhou, D.}, \textsc{Gu, Q.} and \textsc{Szepesvari, C.} (2021).
\newblock Nearly minimax optimal reinforcement learning for linear mixture
  markov decision processes.
\newblock In \textit{Conference on Learning Theory}. PMLR.

\bibitem[{Zimmert and Seldin(2019)}]{pmlr-v89-zimmert19a}
\textsc{Zimmert, J.} and \textsc{Seldin, Y.} (2019).
\newblock An optimal algorithm for stochastic and adversarial bandits.
\newblock In \textit{The 22nd International Conference on Artificial
  Intelligence and Statistics}. PMLR.

\end{thebibliography}
\end{document}